\documentclass[11pt]{article}
\usepackage{graphicx}
\usepackage{endfloat}
\usepackage{amssymb}
\usepackage{amsthm}
\usepackage[colorinlistoftodos]{todonotes}
\usepackage{mathtools}
\usepackage{amsmath}
\usepackage{xcolor}
\usepackage{comment}
\usepackage{JASA_manu}
\usepackage{natbib}

\newtheorem{thm}{Theorem}[section]

\newcommand{\bma}[1]{\mbox{\boldmath $#1$}}
\newcommand{\bX}{ {\bma{X}} }
\newcommand{\bx}{ {\bma{x}} }

\usepackage{Sweave}

\RequirePackage[normalem]{ulem} 
\RequirePackage{color}\definecolor{RED}{rgb}{1,0,0}\definecolor{BLUE}{rgb}{0,0,1} 

\begin{document}
\Sconcordance{concordance:JASA_PDX_revised.tex:JASA_PDX_revised.Rnw:%
1 464 1 1 8 14 1 1 9 1 13 16 1 1 36 1 4 12 1 1 16 1 2 15 1 1 38 1 26 2 %
1 1 49 18 0 1 2 20 1 1 55 1 3 2 1 1 10 1 2 38 1}

 \title{High dimensional precision medicine from patient-derived xenografts} 
 \author{Naim U. Rashid$^{1,2}$, Daniel J. Luckett$^{1}$, Jingxiang Chen$^{1}$, Michael T. Lawson$^{1}$, \\
 Longshaokan Wang$^{7}$, Yunshu Zhang$^{7}$, Eric B. Laber$^{7}$, Yufeng Liu{$^{1,3,4}$}, Jen Jen Yeh{$^{2,5,6}$},\\
 Donglin Zeng$^{1}$, and Michael R. Kosorok{$^{1,4}$}} 
 \maketitle

\noindent$^{1}$Department of Biostatistics

\noindent$^{2}$Lineberger Comprehensive Cancer Center

\noindent$^{3}$Department of Statistics and Operations Research

\noindent$^{4}$Department of Genetics

\noindent$^{5}$Department of Surgery

\noindent$^{6}$Department of Pharmacology

\noindent University of North Carolina at Chapel Hill\\
Chapel Hill, NC, USA\\

\noindent$^{7}$Department of Statistics\\
North Carolina State University\\
Raleigh, NC, USA\\

\clearpage


\begin{center}
\textbf{Abstract}
\end{center}
The complexity of human cancer often results in significant heterogeneity in response to treatment. Precision medicine offers potential to improve 
patient outcomes by leveraging this heterogeneity. Individualized treatment rules (ITRs) formalize precision medicine as maps from the patient covariate space into the space of allowable treatments. The optimal ITR is that which maximizes the mean of a clinical outcome in a population of interest. 
Patient-derived xenograft (PDX) studies permit the evaluation of multiple treatments  within a single tumor and thus are ideally suited for estimating optimal ITRs. PDX data are characterized by correlated outcomes, a high-dimensional  feature space, and a large number of treatments. Existing methods for estimating optimal ITRs do not take advantage of the unique structure of PDX data or handle the associated challenges well. In this paper, we explore machine learning methods for estimating optimal ITRs from PDX data. We analyze data from a large PDX study to identify biomarkers that are informative for developing personalized treatment recommendations in multiple cancers. We estimate optimal ITRs using regression-based approaches such as Q-learning and direct search methods such as outcome weighted learning. Finally, we implement a superlearner approach to combine a set of estimated ITRs and show that the resulting ITR performs better than any of the input ITRs, mitigating uncertainty regarding user choice of any particular ITR estimation methodology.  Our results indicate that PDX data are a valuable resource for developing individualized treatment strategies in oncology. 

\vspace*{.3in}

\noindent\textsc{Keywords}: {Biomarkers, Deep learning autoencoders, 
Machine learning, Outcome weighted learning, Precision medicine, 
Patient-derived xenografts, Q-learning}

\newpage

\section{Introduction}

The complexity of human cancer is reflected in the molecular and phenotypic diversity of 
patient tumors \citep{polyak2011heterogeneity}. This diversity results in heterogeneity 
in response to treatment, which complicates clinical decision making. 
The complexity of human cancer is also reflected in the high failure rate of new therapies entering 
oncology clinical trials, highlighting limitations in the ability 
of standard preclinical models to evaluate new therapies \citep{tentler2012patient}. 
A recent study utilized patient-derived xenografts (PDXs) to perform a 
large-scale screening in mice to evaluate a large number of 
FDA-approved and preclinical cancer therapies \citep{gao2015high}. 
Genomic information and observed treatment responses were used to identify 
efficacious therapies that standard cell line model systems had missed,  
and also validate known associations between genomic biomarkers and differential response 
to treatment. The results from this PDX study mirrored those seen in human patients. 
Thus, PDX models can be used to evaluate \emph{in vivo} therapeutic response and 
discover novel biomarkers to inform individualized treatment decisions.

PDX models are based on the transfer of primary human tumors directly from the patient into an immunodeficient mouse \citep{siolas2013patient}. Briefly, pieces of primary solid tumors are collected from patients by surgery or biopsy \citep{hidalgo2014patient}. The collected tumor pieces from an individual patient are then implanted into mice subcutaneously to create a PDX line, whereby tumor size and rate of tumor growth after implantation may be measured over time. After the tumor reaches sufficient size, the line may be expanded by passaging directly from the implanted tumor into additional genetically identical mice. Through this expansion, multiple treatments may be applied to mice originating from the same PDX line, allowing for the application of multiple treatments to the same patient tumor. High throughput genomic assays such as RNA sequencing (RNA-seq) and DNA sequencing may be performed on the original patient tumor. Features of the original tumor will be retained throughout line expansions \citep{hidalgo2014patient}, making PDX models ideal for learning how to personalize cancer treatment, given observed feature-response associations. 

Personalized treatment recommendations in oncology have traditionally centered around the classification of patients into subgroups \citep{sargent2005clinical}. In some cases, patient subgroups may be derived from predictive models based upon genomic biomarkers \citep{parker2009supervised}. Yet, significant heterogeneity in treatment response may still be observed within such subgroups 
\citep{metzger2012dissecting,chen2014non}, and assignment of optimal treatment  is predicated upon accurate subgroup prediction. An alternative approach to precision medicine is the estimation of an individualized treatment rule (ITR), a map directly from the patient covariate space  to the space of allowable treatments that can be used to make treatment decisions. The optimal ITR is defined as the one that maximizes the mean of a clinical outcome, such as treatment response, when applied in a population.  Examples of such covariates may include patient clinical information, such as laboratory assay results, or high dimensional genomic data, such as gene expression or mutation data from a patient's tumor.  As such, treatment recommendations based upon an optimal ITR may result in improved clinical outcomes  by harnessing individual-specific molecular and clinical features not captured by subgroup-based approaches.  
 
A number of methods have been proposed to estimate an optimal ITR. One approach is to fit a regression model for treatment response given a set of applied treatments and patient covariate information.  The optimal treatment is then the one providing the maximum predicted response in a new patient, given the fitted regression model and the covariate information for that patient \citep{qian2011performance}. An example of this approach is Q-learning \citep{murphy2005generalization,zhao2009reinforcement,qian2011performance,schulte_q_2014}.  Direct search methods, including outcome weighted learning (OWL)  \citep{zhao2012estimating, zhou_residual_2015, liu_robust_2016, chen_estimating_2017}, doubly robust ITR estimators \citep{zhang2012estimating,zhang2013robust},  and marginal structural models \citep{robins2008estimation, orellana2010dynamic} estimate the optimal ITR using inverse probability weighting rather than regression. In direct search methods, the class of ITRs is prespecified, while in 
other approaches, the class of ITRs is implied by the modeling process.  Recent advances in machine learning for causal inference have produced a number of estimators for the conditional average treatment effect that could be used to estimate an ITR for a binary treatment decision 
\citep{imai2013estimating, athey2016recursive, wager2017estimation}. Other methods for estimating optimal ITRs include marginal mean models \citep{murphy2001marginal} and Bayesian predictive  modeling \citep{ma2016bayesian}. Regardless of the approach, many existing methods may directly utilize high-dimensional genomic biomarker data. However, such methods were not designed for PDX studies, where the application of multiple treatments within a subject (PDX line) results in correlated outcomes, and the number of available treatments is large.


In this manuscript, we utilize the wealth of clinical and biomarker data generated by the Novartis PDX study \citep{gao2015high} to estimate optimal ITRs for treating several human cancers. Given the correlated outcomes within PDX lines, the large number of treatments, and the high dimension of the covariate space, it is difficult to fit a nonparametric model to the conditional mean response without large amounts of data. Thus, we explore a number of ways of imposing structure on the conditional mean response, including reducing the dimension of the covariate space, grouping treatments that result in a similar mean response, and constructing a treatment tree to convert the problem of selecting from a large set of treatments to a sequence of binary comparisons. The result is a multi-step procedure, where each step can be thought of as imposing structure on the model for the conditional mean response in a way that alleviates the challenges posed by PDX data and takes advantage of the unique structure of PDX data. Such multi-step procedures are well-studied and have shown good performance in many applications \citep{deseq2, bourgon2010independent, soneson2013comparison, rashid2014differential}.  We show that the proposed multi-step procedure achieves improved performance over standard methods in certain settings. We examine the various modeling decisions that are made at each step of the multi-step procedure and demonstrate the use of super-learning  \citep{luedtke2016super} to improve prediction performance by aggregating the proposed models.

In Section~\ref{data}, we describe the large-scale PDX data set that motivated this work and use it to highlight the challenges associated with estimating optimal ITRs using PDX data. We present our methodological approaches in Section~\ref{methods}. We present results from our data analyses in Section~\ref{results}. In Section~\ref{discussion}, we conclude with a discussion to compare and contrast the various modeling decisions we made and discuss the clinical implications of our findings. Additional details, including software for reproducing our work, are given in the Supplemental Materials. 

\section{Large-Scale PDX Drug Screen for Treatment Response} \label{data}

\subsection{Data Overview}

The Novartis PDX study \citep{gao2015high} established a total of 1075 PDX lines corresponding to a variety of human cancers. A subset of these lines were genomically profiled prior to treatment for gene expression (399 lines), copy number analysis (375), 
and mutations (399). In addition, 281 lines were enrolled in the drug response study. 
Our study utilizes 190 PDX lines with complete genomic and response data (Supplementary Figure 1). 
Five types of cancer are represented among these 190 lines (Supplementary Figure 2): 
breast cancer (BRCA), melanoma (CM), colorectal cancer (CRC), non-small cell lung 
cancer (NSCLC), and pancreatic cancer (PDAC). 
A median of 21 treatments were applied per PDX line across different cancers (Supplementary Figure 2). 
Certain PDX lines had fewer mice and therefore fewer treatments than the total 
number of treatments available for a particular cancer type. One mouse per PDX line was set as a 
control line and did not receive treatment. 

In total, 3487 mice, expanded from the 190 PDX lines with complete data, were 
used for our study. Details regarding the biomarker data utilized 
in this study are given in Section 6 of the 
Supplemental Materials. Briefly, each of the 190 PDX lines had 22,665 genes 
measured for gene expression via RNA-seq, 23,854 genes measured for gene-level 
copy number estimates via copy number array, and between 159 and 293 
mutations (25$^{th}$ and 75$^{th}$ percentiles)  identified via DNA sequencing. 
In total, a union set of approximately 60,000 features are available for ITR 
estimation. Because each genomic assay was performed on the patient tumor prior 
to implantation, all mice expanded from the same PDX line share the same set of genomic biomarkers.  

\subsection{Study Design and Response Variables}

Mice from each PDX line were treated with either single agents or combinations. 
A total of 38 unique therapies were applied, administered either as a single 
agent (36 total administered) or in combination with other agents 
(26 total combinations administered). Certain treatments were limited to particular 
cancers, whereas others were applied across cancers. 
Each mouse was treated for a minimum of 21 days. 
Tumor size was evaluated twice weekly by caliper measurements, and the 
approximate volume of the tumor was calculated using the formula 
($l\times w\times w)\times(\pi/6)$, where $l$ is the major tumor axis and 
$w$ is the minor tumor axis. 

	
Two measures were utilized to summarize response to treatment: 
best average response (BAR) and time to tumor doubling (TTD). BAR is defined as 
$\min_{t: d_t > 10} \left[\{1/(t+1)\}\sum_{l=0}^{t}
(V_l - V_{0})/V_{0}\right] \times 100\%$, 
where $d_t$ is the day in which the $t$th measurement was taken, $V_{0}$ is tumor 
volume at day 0, and $V_{l}$ is tumor volume at measurement $l$ taken 
on day $d_l$ \citep{gao2015high}. BAR is a measure of the maximum observed tumor shrinkage from 
baseline over measurements taken at least 10 days after start of treatment, 
scaled by time since baseline. More negative values of BAR indicate 
a better response. We used $-$BAR in the analysis so 
that larger values indicate a better response.
BAR mirrors similar criteria for assessing response in human clinical 
trials \citep[RECIST,][]{therasse2000new}. TTD is defined as the number of 
days from baseline that the tumor doubled in size from its baseline measurement. 
Due to skewness in the distribution of TTD, we used the natural log for analysis. 

\subsection{Implications of PDX Data for ITR Estimation}

The unique structure of PDX data---the application of multiple treatments to mice implanted with 
the same tumor---makes PDX data ideally suited for estimating optimal ITRs. Comparing responses 
between mice within the same line does not amount to observing true treatment effects due 
to the inherent variability that exists across mice; however, the improved precision of observing responses to multiple treatments applied to the same tumor 
may substantially improve the performance of estimated ITRs. The method we propose involves a 
sequence of decisions between two groups of treatments. At each step, we aim to choose the group 
that contains the best treatment; thus, the proposed method directly uses multiple responses within each line that would not be available without PDX data. 

Existing methods for estimating optimal ITRs have typically been designed for a small number of treatments (two, for example). In this case, the set of treatments is large (20 or greater). Modeling the conditional response is difficult in the presence of a large set of treatments due to 
the large number of terms that would be included in the model. In a preclinical study like this, it is more important to
identify groups of promising treatments than to fully assess all pairwise comparisons between treatments.
Thus, we reduce the size of the treatment set by grouping treatments with similar mean responses using hierarchical 
clustering. This allows us to adaptively create treatment groups which have approximately similar treatment effects.

Finally, the set of genomic biomarkers is high-dimensional. Approximately 60,000 genomic features 
are available, and many of the available features may exhibit either low variability or low expression 
across samples. While some methods for estimating an optimal ITR can handle such high dimension, we implement some common preprocessing steps similar to existing statistical methods in bioinformatics to reduce the dimension prior to ITR estimation.  We also utilize some additional steps to mitigate the ultra-high dimension of the predictor space, which we motivate and justify in the next section.  

\section{Methods} \label{methods}

\subsection{Overview}

Let $Y$ denote response and let $\bX$ denote a vector of covariates. Let $A = (A_1, \ldots, A_J)$ denote treatment, 
where exactly one of $A_1, \ldots, A_J$ is equal to 1 and the rest equal to 0, with $A_j = 1$ 
indicating that treatment $j$ is given. The conditional mean response can be expressed as
\begin{equation} \label{model}
E\left(Y | \bX = \bx, A = a\right) = h_0(\bx) + \sum_{j = 1}^J h_j(\bx) a_j.
\end{equation}
If we were to obtain estimators $\widehat{h}_j$, $j = 1, \ldots, J$, an estimator for the optimal ITR would be 
$\widehat{d}(\bx) = \operatorname*{arg \, max}_{j = 1, \ldots, J} \widehat{h}_j(\bx)$. However, 
fitting model~(\ref{model}) nonparametrically in the PDX setting would be difficult without large amounts of data, due to the 
large number of treatments and high dimension of the covariate space.  Therefore, we propose to impose structure on model~(\ref{model}) 
to ameliorate these difficulties. The result is a multi-step procedure to reduce the dimension of the feature space, 
reduce the size of the treatment set, and estimate optimal ITRs for the reduced treatment set 
using the reduced feature space. The steps of our procedure are described in Sections~3.1.1-3.1.3 below. 
Various modeling decisions must be made at each step, and some alternative choices for these decisions 
are discussed in Section~\ref{variants}. Because it is not immediately obvious what the optimal set of 
modeling decisions is, we apply super-learning \citep{luedtke2016super} to combine 
variants of the proposed method with different embedded models. 
For complete details on our methodology, we refer the reader 
to Section 7 of the Supplemental Materials.

\subsubsection{Preprocessing (step 1).} \label{preprocess}

In the first step, we perform a preprocessing of genomic features to remove 
features without sufficient variance (Supplemental Materials Section 7.1). 
Genomic features are filtered out at this step due to low expression or low variability. This 
is a well-studied technique for dimension reduction prior 
to analysis \citep{deseq2,bourgon2010independent,soneson2013comparison,rashid2014differential}. 
This screening is performed separately for 
each cancer type. A summary of the number of genes and 
features remaining after this step is given in Supplementary Table 2. 
In addition, treatments applied in 
less than 90\% of PDX lines within a cancer type were filtered out for that cancer type 
(see Section 7.1 of the Supplemental Materials). 
We summarize the number of treatments applied per PDX line per cancer after treatment 
filtering in Supplementary Table 1. 

\subsubsection{Supervised screening (step 2).} \label{screen}

In the second step, we further reduce the dimension of the feature space by 
selecting likely prognostic and predictive genomic features 
(Supplemental Materials Section 7.2). Genes are ranked using 
Brownian distance covariance \citep{szekely2009brownian}, evaluating the 
dependence between a vector of gene-level predictors for each PDX line and 
the bivariate response (BAR, TTD) for a given treatment (prognostic) or 
difference in response between a pair of treatments (predictive). After 
ranking genes, we select the top $L_{\text{SUP}}$ genes 
($L_{\text{SUP}} = 50, 100, 500, 1000$, Supplementary Table 2) and use all 
available platforms for the selected genes, giving us $p_{L_{\text{SUP}}}$ 
features corresponding to each value of $L_{\text{SUP}}$. 
Screening in this way imposes structure on model~(\ref{model}) by forcing the 
$h_j(\bx)$ to be constant in some of the covariates for certain $j = 0, \ldots, J$.
This strategy helps to alleviate the difficulties caused by the large number of genomic features and is similar 
in nature to sure independence screening \citep{fan2008sure}, used in ultra-high-dimensional 
regression problems to reduce the feature dimension to a more moderate 
size prior to application of variable selection methods. 
All of our analyses are repeated for each of $L_{\text{SUP}} = 50, 100, 500, 1000$ 
top genes to yield insights into the performance of estimated ITRs based on differing numbers of features. 
Screening is performed separately for each cancer type. 

Because the set of genomic features consists of multiple platforms per gene, 
the feature set resulting from selecting the top $L_{\text{SUP}}$ genes may still 
contain a large number of features. In addition, predictors generated from different 
platforms (e.g., RNA-seq and copy number) on the same gene may be correlated, and 
gene expression may be correlated across different genes. 
Thus, there may be a lower-dimensional feature space that 
would be sufficient for ITR estimation. To address this, we 
applied a further dimension reduction step using deep learning autoencoders 
\citep[DAE,][]{wang2017adnn}, a variant of deep neural networks 
\citep{vincent_stacked_2010}. 
See Section 7.3 of the Supplemental Materials for details. 
This allows us to evaluate whether a lower dimensional 
representation of the feature set improves ITR estimation at the cost 
of additional computational burden. 

Deep learning autoencoders build a 
nonlinear prediction model to predict all feature variables from a 
low-dimensional representation of the original features, with dimension 
chosen by cross-validation. The reconstruction error of this approach is 
several orders of magnitude lower than principal components analysis across cancer types, 
particularly for $L_{\mathrm{SUP}} = 50, 100$ (Supplementary Tables 3 and 4). 
This indicates that linear dimension reduction techniques are not 
sufficient to capture the information contained in the data. 
Dimension reduction using DAEs imposes structure on model~(\ref{model}) by forcing the 
$h_j(\bx)$, $j = 0, \ldots, J$ to depend on $\bx$ only through the low-dimensional summary 
of the original features. The dimensions for 
each feature set following application of deep learning autoencoders within 
each cancer are given in Supplementary Table 5.
After applying steps 1 and 2 to obtain low-dimensional sets of covariates, 
we use the low-dimensional sets of covariates to estimate optimal ITRs.

\subsection{Estimation of Treatment Tree-based ITRs}\label{tree_overview}

Estimating an ITR to select from a large number of treatments is challenging for two reasons. 
First, fitting model~(\ref{model}) in the presence of a large number of treatments would be difficult 
due to the large number of treatment $\times$ feature interaction terms. Second, the resulting ITR would be difficult to interpret 
and implement. Instead of an ITR which selects one treatment from the full set of treatments, it would be more useful to 
have an ITR which selects a group of treatments from a partition of the full set of treatments, 
where treatments in the same group can be expected to lead to similar responses. 
Ideally, we would create groups of treatments such that the conditional mean response function is 
the same for all treatments in the same group. However, this treatment grouping is unknown. 
In Section~\ref{tree}, we describe a heuristic technique  
for approximating this treatment grouping using hierarchical clustering. This allows us to create different 
groupings of treatments by cutting the tree (the dendogram resulting from the clustering) 
in different places. We then define a class of ITRs that can be expressed as a sequence 
of binary decisions, one for each step of the tree. 
This tree structure is distinct from tree-based regimes which use a tree 
structure to represent the final estimated ITR \citep{laber2015tree,zhang2015using}.
A variety of methods could be used to 
construct decision rules for each step of the tree. We introduce two: a regression-based approach 
(Section~\ref{qlearn}) and a direct search approach (Section~\ref{owl}).

\subsubsection{Treatment tree construction.} \label{tree}

Denote the $i^{th}$ mouse corresponding to the $j^{th}$ PDX line within the $k^{th}$ 
cancer type with subscript ${ijk}$, where  $k = 1,\ldots, 5$, $j = 1,\ldots, m_k$, 
and $i = 1,\ldots,p_{jk}$, with $p_{jk}$ representing the number of treatments applied 
to PDX line $j$ in cancer type $k$, and $m_k$ representing the number of PDX lines 
for cancer type $k$ (Supplementary Figure 2). We let 
$P_k = \max_{j = 1, \ldots, m_k} p_{jk}$, i.e., for each PDX line, up to $P_k$ 
mice were expanded to receive a maximum of $P_k$ treatments in cancer type $k$, 
one treatment applied per mouse (Supplementary Figure 3). For certain lines, $p_{jk} < P_k$ 
treatments may have been applied, as the number of mice per line 
varied. For ease of notation, we assume that $p_{jk} = P_k$ 
and that the $i$th mouse in each PDX line received the same treatment within cancer type.  

For treatment $i$ in cancer $k$, we define the treatment response vector 
$Y_{ik} = (Y_{i1k}, \ldots, Y_{im_kk})^\intercal$, where $Y_{ik}$ is scaled to have standard 
deviation 1 (rows in Supplementary Figure 6, left panel). Because there are inherent 
baseline differences in response between PDX lines, we first center the response 
values of each PDX line (columns in Supplementary Figure 6, left panel) with respect 
to the ``null'' response within each PDX line. 
Using PDX lines with $P_k$ mice, we calculate the Euclidean distances between $Y_{ik}$ 
and $Y_{i^\prime k}$, $i, i^\prime = 1, \ldots, P_k$, $i \ne i^\prime$ to get a 
distance between each pair of treatments. For a fixed constant, $c_1$, 
we group the $c_1$ nearest neighbors of the ``untreated'' treatment response 
vector  to form a ``null'' set of treatments, denoted by $A_{k, c_1, 0}$, containing 
those treatments producing low or no response. 
Then, for each $i \notin A_{k, c_1, 0}$, we compute the $m_k \times 1$ vector of centered observed outcomes, 
$R_{ik} = Y_{ik} - \bar{Y}_{A_{k, c_1, 0}}$, where $\bar{Y}_{A_{k, c_1, 0}}$ is the $m_k \times 1$ 
vector of averaged treatment responses for each of the treatments in $A_{k, c_1, 0}$. 
Thus, $R_{ik}$ is the difference between the response to treatment $i$ 
and the average response in the null treatment group.

Next, we perform hierarchical clustering on the centered treatment response vectors and 
take the resulting dendogram as a treatment tree (Supplementary Figure 6, right panel). 
For a fixed constant, $c_2$, we can construct treatment groups by cutting the 
tree $c_2$ steps from the root node. For cancer type $k$, we label the treatment groups determined by 
$c_1$ and $c_2$ as $A_{k, c_1, 1}, \ldots, A_{k, c_1, c_2 + 1}$ and denote the set of all 
possible treatment groups in cancer $k$ given $(c_1, c_2)$ by 
$\mathcal{A}_{k, c_1, c_2} = \displaystyle{\bigcup_{a = 1}^{c_2 + 1}}A_{k, c_1, a}\cup A_{k, c_1, 0}$. 
Although $\mathcal{A}_{k, c_1, c_2}$ depends on $c_1$ and $c_2$, 
we will write $\mathcal{A}$ in place of $\mathcal{A}_{k, c_1, c_2}$ to simplify 
notation. The hierarchical clustering groups treatments together when 
they result in similar mean responses, in contrast to grouping treatments by 
predefined characteristics such as molecular target or mechanism of action. 
Grouping treatments in this way forces certain $h_j$, $j = 1, \ldots, J$ in model~(\ref{model}) to 
be equal, similar to a fusion penalty \citep{tibshirani2005sparsity}. 
This strategy helps to alleviate the difficulties caused by the large number 
of treatments and allows the estimated ITRs to select a group of treatments likely 
to produce similar outcomes \citep{wu2016set}. 
We average responses within the resulting treatment groups, 
yielding one response for each PDX line and treatment group in $\mathcal{A}$. 
An ITR can be estimated for each value of $c_1$ and $c_2$, and the optimal values 
of $c_1$ and $c_2$ can be selected using cross-validation. 

\subsubsection{Identification of the optimal ITR.} \label{optimal_ITR}

Let $\bX\in \widehat{\mathcal{X}} \subset\mathbb{R}^{p_{L_{\text{SUP}}}}$ be the 
vector of genomic features, where we write $\widehat{\mathcal{X}}$ to make it 
clear that the domain of the feature space is data-dependent 
(see Sections~\ref{preprocess} and \ref{screen}). 
An ITR is a mapping $D: \widehat{\mathcal{X}} \rightarrow \mathcal{A}$. For each treatment 
$a \in \mathcal{A}$, define the potential outcome $R^*(a)$ to be the outcome 
that would be observed under treatment $a$ \citep{rubin1978bayesian}. 
Within cancer type $k$, the set of potential outcomes consists of 
$R^*_{ij}(a_i)$, $i = 1, \ldots, P_{k}$, $j = 1, \ldots, m_k$. 
%
We make the assumption that $E\left\{R^*_{ij}(a)\right\} = E\left\{R^*_{i^\prime j}(a)\right\}$ 
for all $i, i^\prime = 1, \ldots, P_{k}$, all $j = 1, \ldots, m_k$, and all $a \in \mathcal{A}$. 
That is, we assume that the expected value of the 
potential outcomes for mouse $i$ and $i'$ from the same PDX line 
would be equal if they both received the same treatment. 
The standard assumption of positivity is not needed because each PDX line is assigned to 
every treatment used for that cancer type. On the other hand, while the standard assumption of no unmeasured 
confounders is still needed in our setting, the primary source of confouding is due to the assignment of mice
to treatment within PDX line and not due to the genetic features of the PDX line since each line receives
all treatments. Hence we assume that the process of assigning treatment to mice is exchangeable and homogeneous,
and thus the no unmeasured confounding assumption obtains. 
Define the value of an ITR, $D$, by $V(D) = E \left(E\left[R^*\left\{D(\bX)\right\} | \bX\right]\right)$. 
Let $\mathcal{D}$ be the set of ITRs which can be expressed 
as a sequence of decision rules, one for each step of the treatment tree, starting from the root node and 
proceeding until a leaf node is reached. The optimal ITR associated with the treatment tree is 
$D^{*}=\underset{D\in\mathcal{D}}{\arg\max} V(D)$, i.e., the ITR that maximizes value within the class. 

\subsubsection{Treatment tree-based Q-learning.}\label{qlearn}

In this section, we propose an extension of Q-learning 
\citep{murphy2005generalization, zhao2009reinforcement, zhao_reinforcement_2011, schulte_q_2014} 
to estimate the optimal ITR associated with the treatment tree. 
Let $a_w(t)$, $w = 0,1$ denote the set of treatment groups downstream to the 
left ($a_1(t)$) and right ($a_0(t)$) arms of a node at step $t$, where 
$t = c_2, \ldots, 0$. Starting at the step corresponding to the lowest node, 
we compute $R_{jw}(c_2) = \left[\sum_{i = 1}^{P_k} R_{ij} 
I\left\{i \in a_w(c_2)\right\}\right] / \sum_{i = 1}^{P_k} 
I\left\{i \in a_w(c_2)\right\}$ for $j = 1, \ldots, m_k$ and $w = 0, 1$. 
Here $R_{jw}(c_2)$ is the mean observed reward (centered response) in PDX 
line $j$ across the treatments belonging to treatment group $a_w(c_2)$. 
We then fit a regression model for $R_{jw}(c_2)$ based on $\bX_{jk}$, separately 
for $w = 0$ and $w = 1$, to obtain $\widehat{E}_{c_2}\left[R^*\left\{a_w(c_2)\right\} | \bX = \bx\right]$, 
the estimated conditional mean reward in treatment group $w$ at step $c_2$ given 
genomic features. If either $R_{j0}(c_2)$ or $R_{j1}(c_2)$ 
are missing for a given line, then that line did not receive that particular treatment; 
these observations are removed before fitting the regression model. 
The estimated optimal decision rule at step $c_2$ of the tree for an 
individual with genomic features $\bx$ is given by
\begin{equation}
\widehat{D}^{QL}_{c_2}(\bx) = \operatorname*{arg max}_{w \in \{0, 1\}}
\widehat{E}_{c_2}\left[R^*\left\{a_w(c_2)\right\} | \bX = \bx\right], \label{eq:DQL}
\end{equation}
i.e., the treatment group with the highest predicted clinical response given 
$\bx$ is the estimated optimal treatment group at step $c_2$. 
We repeat the above process for step $t = c_2 - 1, \ldots, 1$, except $R_{jw}(t)$ only utilizes 
the observations for each PDX line from the optimal treatment group selected at 
the previous step. Thus, the proposed method directly utilizes the multiple responses 
observed for each PDX line. 

Step $t = 1$ corresponds to the highest node in the treatment 
tree consisting of the non-null treatments, and step $t = 0$ corresponds to the 
split between the null group (centered by their own means) and the non-null treatments 
from the treatment tree. The decision rule is determined in a similar manner at 
step $t = 0$, evaluating whether any non-null treatment should be applied. 
At $t = 0$, the response vector for the null group is a vector of zeros, 
and the decision rule at this step simplifies to determining whether the expected value 
of the response under the optimal non-null treatment is greater than 0. 
A number of techniques could be used to fit the regression models at each stage, 
representing different ways of imposing structure on the $h_j$, $j = 0, \ldots, J$ in model~(\ref{model}). 
We discuss various embedded models that could be used in more detail in Section~\ref{variants}. 
 
Given the sequence of estimated decision rules 
$\left(\widehat{D}^{QL}_0, \ldots, \widehat{D}^{QL}_{c_2}\right)$, 
the optimal treatment for a new individual given their set of predictors 
is obtained by following the decision rules sequentially from step $t = 0$ 
downward until one arrives at a terminal node on the tree. Rather than recommending 
a sequence of treatments, as in standard Q-learning, the estimated optimal 
decision rule at each step of the tree directs the user through either the 
left or right arm at each node, creating a path through the tree to arrive at a 
terminal node representing the optimal treatment group for that individual.

\subsubsection{Treatment tree-based outcome weighted learning.} \label{owl}

Outcome weighted learning (OWL) estimates the optimal ITR for selecting between two treatments 
by maximizing an inverse probability weighted estimator of the value function over a 
fixed class of decision rules. Thus, unlike Q-learning, OWL does not rely a fitted regression model. 
Our extension of OWL adopts the same tree structure as in the previous section. However, 
instead of fitting a separate regression model for each arm at split $t$, the optimal decision rule 
at split $t$ is estimated directly using weighted classification methods 
\citep{zhou_residual_2015, liu_robust_2016, chen_estimating_2017}. 
When treatment is binary, any decision rule can be written as 
$D(\bX) = \mathrm{sign}\left\{f(\bX)\right\}$ for some decision function, $f$. 
We will assume that the decision boundary at each split is smooth. 
Thus, at each split, we use a class of decision rules defined by some class of smooth functions $\mathcal{F}$, 
e.g., a reproducing kernel Hilbert space. 
At the $t$th split of the treatment tree, let 
$\widehat{D}_t^{OWL}(\bx) = \mathrm{sign}\left\{\widehat{f}_t(\bx)\right\}$, where 
\begin{equation} \label{eq:go-learning}
\widehat{f}_t = \underset{f\in\mathcal{F}}{\arg\min}\mathbb{E}_{nt}
\left(|R|\left[I(R\geq 0)\{1 - Af(\bX)\}_{+} + I(R<0)\{1 + Af(\bX)\}_{+}\right]\right) + \lambda_n J(f).
\end{equation}
Here, $\mathbb{E}_{nt}$ denotes the empirical measure of the data used 
in the $t\mathrm{th}$ split, $J(f)$ is a penalty term for the decision function $f$, 
$\lambda_n$ is a tuning parameter which is selected using cross-validation, 
and $\mathcal{F}$ is a space of functions. As in Q-learning, the observations included 
when computing the minimizer in equation~(\ref{eq:go-learning}) are all of those corresponding 
to the estimated optimal treatment group at the previous step. Thus, our extension of OWL 
also directly utilizes the multiple outcomes observed per PDX line. We discuss options for 
selecting $\mathcal{F}$ in Section~\ref{variants}; different options for $\mathcal{F}$ 
impose different specific forms that the $h_j$, $j = 0, \ldots, J$ must take in model~(\ref{model}).

\subsubsection{Theoretical justification of tree-based ITRs.}

In this section, we show that the value of the optimal ITR associated with a treatment tree 
increases strictly with each step down the tree, until the step where there is no 
heterogeneity in conditional mean response for treatments within the same group. 
Let ${\cal G}=\{G_1, \ldots, G_p\}$, $1 < p \leq P_k$, be a grouping of treatments, that is, 
${\cal G}$ is a partition of the set of all $P_k$ treatments.
Conditioning hereafter on the selected set of features $\widehat{\cal X}$,
let $D:\widehat{\mathcal{X}} \rightarrow {\cal G}$ be an ITR, i.e., $D(\bx) = G$ implies 
that treatment group $G$ is selected for a subject with features $\bX = \bx$. 
We assume that if $D(\bx) = G$, the decision maker will select one of the treatments from 
$G$ with equal probability. 
Define the value of $D$ with respect to ${\cal G}$ to be
$$
V(D; {\cal G}) = E\left[ \sum_{G \in {\cal G}} 1\left\{D(\bX) = G\right\} |G|^{-1} \sum_{a \in G}  
E\left\{ R^*(a) | \bX\right\} \right],
$$
where $|G|$ is the number of treatments in $G$. The optimal ITR is given by 
$$
D^*(\bx) = \operatorname*{arg \, max}_{G \in \mathcal{G}} |G|^{-1} \sum_{a \in G} E\left\{ R^*(a) | \bX = \bx\right\} 
$$
and the value of the optimal ITR (also called the optimal value) is 
$$
V^*\left({\cal G}\right) = E \left[ \operatorname*{max}_{G \in \mathcal{G}} |G|^{-1} \sum_{a \in G} 
 E\left\{ R^*(a) | \bX\right\} \right].
$$
We call a grouping (partition) $\widetilde {\cal G} = \left\{\widetilde G_1, \ldots, \widetilde G_m \right\}$ 
finer than ${\cal G}$ if any $G \in {\cal G}$ can be written as the union of sets in $\widetilde {\cal G}$. 
We obtain the following result:
\begin{thm} \label{thm1}
If $\widetilde{\cal G}$ is finer than ${\cal G}$, 
then $V^*\left(\widetilde {\cal G}\right) \ge V^*\left({\cal G}\right)$. 
Furthermore, equality holds if and only if, for any $\widetilde{G} \in \widetilde{{\cal G}}$ and $G \in {\cal G}$ 
such that $\widetilde{G} \subset G$, 
$\left|\widetilde{G}\right|^{-1} \sum_{a \in \widetilde{G}} E\left\{ R^*(a) | \bX\right\}
 = |G|^{-1} \sum_{a \in G} E\left\{ R^*(a) | \bX\right\}$ with probability one. 
\end{thm}
\begin{proof}
For $G \in {\cal G}$ with $G = \cup_{j = 1}^m \widetilde G_j$, 
we have for any $\bx \in \widehat{\mathcal{X}}$ that 
\begin{eqnarray*}
|G|^{-1} \sum_{a \in G} E\left\{ R^*(a) | \bX = \bx\right\} & = &
 |G|^{-1} \sum_{j = 1}^m \sum_{a \in \widetilde{G}_j} E\left\{ R^*(a) | \bX = \bx\right\} \\
 & = & \sum_{j = 1}^m \left(\left|\widetilde{G}_j\right| / |G|\right) \left|\widetilde{G}_j\right|^{-1} 
 \sum_{a \in \widetilde{G}_j} E\left\{ R^*(a) | \bX = \bx\right\} \\
 & \le & \operatorname*{max}_{j = 1, \ldots, m} \left|\widetilde{G}_j\right|^{-1} 
 \sum_{a \in \widetilde{G}_j} E \left\{ R^*(a) | \bX = \bx\right\}.
\end{eqnarray*}
The above holds with equality if and only if 
$\left|\widetilde{G}_j\right|^{-1} \sum_{a \in \widetilde{G}_j} E\left\{ R^*(a) | \bX\right\}$ 
is the same for all $j = 1, \ldots, m$, or
equivalently, $\left|\widetilde{G}_j\right|^{-1} \sum_{a \in \widetilde{G}_j} E\left\{ R^*(a) | \bX\right\}
 = |G|^{-1} \sum_{a \in G} E\left\{ R^*(a) | \bX\right\}$ for all $j = 1, \ldots, m$, with probability one. 
\end{proof}
%
Theorem~\ref{thm1} states that a finer partition of treatments never leads to a decrease in  
value of the optimal ITR and that the change in value of the optimal ITR is zero only when 
the conditional mean response is constant across treatments within groups in the finer partition. 

\begin{thm} \label{thm2}
For a grouping ${\cal G}$ and any set $G \in {\cal G}$, further partition $G$ into two nonempty groups 
$G_{1}\cup G_{2}$ to obtain a finer partition, $\widetilde{\cal G}$. 
If for any such finer partition, $\widetilde{\cal G}$, 
$V^*\left(\widetilde {\cal G}\right) = V^*({\cal G})$, it holds that 
$$
E \left\{R^*(a_1) | \bX\right\} = E \left\{R^*(a_2) | \bX\right\}
$$
with probability one for all $a_1, a_2 \in G$. 
\end{thm}

\begin{proof}
From Theorem~\ref{thm1}, we obtain that for any $G_{1}$ and $G_{2}$ such that 
$G = G_{1}\cup G_{2}$, 
$$
\left|G_1\right|^{-1} \sum_{a \in G_1} E\left\{R^*(a) | \bX \right\} 
 = \left|G_2\right|^{-1} \sum_{a \in G_2} E\left\{R^*(a) | \bX \right\},
$$
with probability one. Suppose that the proposition does not hold. 
Then, there exists some $a_1, a_2 \in G$ such that $E\left\{R^*(a_1) | \bX = \bx_0\right\} 
\ne E\left\{R^*(a_2) | \bX = \bx_0\right\}$ for all $\bx_0 \in \mathcal{X}_0$ for some 
$\mathcal{X}_0 \subset \widehat{\mathcal{X}}$ with positive probability. 
For any $\bx_0 \in \mathcal{X}_0$, order the $a_i \in G$ such that 
$E\left\{R^*(a_i) | \bX = \bx_0 \right\}$ are nondecreasing with at least two adjacent values 
different, say positions $l$ and $l + 1$ in the ordered sequence. Then, if we let $G_1 = \left\{a_1, \ldots, a_l\right\}$ 
and $G_2 = \left\{a_{l + 1}, \ldots, a_q\right\}$, we obtain that 
$$
|G_1|^{-1} \sum_{a \in G_1} E\left\{R^*(a) | \bX = \bx_0\right\} 
 > |G_2|^{-1} \sum_{a \in G_2} E\left\{R^*(a) | \bX = \bx_0\right\},
$$ 
for all $\bx_0 \in \mathcal{X}_0$. Because $\mathcal{X}_0$ has positive probability, this creates a contradiction. 
\end{proof}
Theorem~\ref{thm2} implies that if no binary splits at one branch will lead to an increase in the optimal value, 
then all treatments within that branch are homogeneous with respect to conditional mean response. Thus, each further 
partition created by cutting the tree at a lower step will lead to an increase in the optimal value function 
until there is no heterogeneity in response across treatments in the same group. 

\subsection{Modeling Decisions} \label{variants}

The framework proposed in this paper involves creating a tree of treatments using hierarchical clustering 
and estimating an ITR as a sequence of binary decisions, one for each step of the tree. 
Within this general framework, there are a number of specific 
modeling decisions that need to be made. In this section, we describe different variants of the proposed 
method that result from making different modeling decisions. 

The tree-based ITR estimation method involves an embedded method to estimate the optimal decision 
rule at each step. Fitting a regression model at each step results in an extension of Q-learning 
(see Section~\ref{qlearn}). A number of regression models could be used as the embedded regression 
model; in our analyses, we used linear models with a LASSO penalty and random forests. 
In Section~\ref{results}, these are referred to as $\text{QL}$ and $\text{QL}_\text{RF}$, respectively. 
At each split, our extension of Q-learning uses the maximum outcome across treatments downstream 
to the left and right of the split as the observed outcomes when fitting the regression model. Alternatively, we 
could obtain predicted maximum outcomes using the regression model and use 
the predicted maximum outcomes to fit the model at the next step. 
These are analogous to the pseudo-values used in standard versions of Q-learning 
\citep{zhao2009reinforcement, zhao_reinforcement_2011}, and would allow application of the proposed 
tree-based Q-learning method to data that do not come from a PDX study. 
We examine both strategies in Section~\ref{results}, and we refer to 
Q-learning with observed outcomes and with pseudo-values as $\text{QL}_1$ and $\text{QL}_2$, respectively. 

Constructing an inverse probability weighted estimator (IPWE) at each 
step, rather than fitting a regression model (as in Q-learning), results in an extension of OWL (see Section~\ref{owl}). 
The IPWE is maximized over a class of functions. We use both the class of linear functions and the 
reproducing kernel Hilbert space associated with the Gaussian kernel function 
\citep{zhao2012estimating, chen_estimating_2017}. In Section~\ref{results}, these are referred to as 
$\text{OWL}_{\text{linear}}$ and $\text{OWL}_{\text{kernel}}$, respectively. 

The second modeling decision that must be made is the selection of the covariate set. 
We applied the proposed method using the $p_{L_{\text{SUP}}}$ dimensional set of genomic features 
resulting from screening using Brownian distance covariance, for each $L_\text{SUP} = 50, 100, 500, 1000$ (Supplementary Table 1). 
We also applied the proposed method to the lower dimensional set of features extracted from 
$L_\text{SUP}$ genes using the DAE, for each $L_\text{SUP} = 50, 100, 500, 1000$  (Supplementary Table 5). 
Analyses using the features extracted from the DAE are labeled with the subscript ``dl" 
in Section~\ref{results}. 

A final variant of the method that we explored involves replacing the observed 
outcomes with model-predicted outcomes before estimation. We fit a random forest to predict outcomes 
based on covariates alone (features and treatments) and replaced the observed outcomes with 
predictions based on the fitted model for all later stages of the analysis. This approach 
acts to denoise the observed outcomes. Analyses utilizing this approach are labeled with the 
subscript ``smoothed". In Section~\ref{results}, we report analyses using different combinations 
of the modeling decisions described above to capture synergistic effects of the 
various modeling decisions. In addition to estimating tree-based ITRs as proposed in 
Section~\ref{tree}, we also estimated ITRs by fitting model~\ref{model} using linear models with the LASSO 
penalty and random forests. These ``off-the-shelf" methods were included to compare to 
the proposed method. 

\subsubsection{Super-learning.} \label{sl}

The various modeling decisions outlined in Section~\ref{variants} above result 
in many variants of the proposed method. While certain variants may work better 
than others in certain settings, the optimal choice of modeling decisions may not 
always be clear. A natural analysis to try in this case is to combine a set of input ITRs 
estimated using various embedded models in the hopes that the resulting ITR performs 
better than any of the input ITRs. To accomplish this, we apply the super-learning algorithm of \cite{luedtke2016super}.  Briefly, \cite{luedtke2016super} propose combining a number of existing methods using  cross-validation to calculate a linear combination of latent functions to maximize the value function. In our implementation, there is not an explicit latent function due to the fact that  we are using a sequential treatment tree as our model. To approximate the value of the latent function with respect to a single treatment, we utilize the predicted reward from one of our sub-models at a given node in the tree whose direct children include our goal treatment. To optimize the superlearner,  we use simulated annealing to estimate the coefficients. Multiple chains are used whose starting points are selected from a set of randomly generated coefficients.  In Section~\ref{results}, variants of super-learning with different sets of input ITRs are referred to with the subscript $\text{sl}$, followed by the number of ITRs that are combined to produce the superlearner.

\subsection{Performance Measures for Individualized Treatment Rules} \label{performance}

We used five fold cross-validation to evaluate the performance of ITRs estimated using 
the proposed method with the various modeling decisions as described in Section~\ref{variants}. 
Within each cancer type, we divided PDX lines into five folds. An optimal ITR was estimated 
using the training data set that results from holding out each fold, and the estimated value was calculated on 
the held-out fold as $\mathbb{E}_n \left[ R I\left\{A = \widehat{D}^*\left(\bX\right)\right\} \right] 
/ \mathbb{E}_n\left[I\left\{A = \widehat{D}^*\left(\bX\right)\right\}\right]$, where $\mathbb{E}_n$ 
denotes the empirical measure taken over mice in the held-out fold. 
Tuning parameters (such as $c_1$ and $c_2$ for estimating 
tree-based ITRs and the OWL penalty parameter) were selected using cross-validation within 
each training data set. The estimated value of the estimated ITR, denoted by $\bar{V}\left(\widehat{D}^*\right)$ 
is calculated by averaging the value estimates resulting from each fold. 
We also calculated the standard deviation of value estimates across folds. 

Different cancer types may result in different marginal mean outcomes. To facilitate comparisons across 
cancer types, we also computed the observed value, $V_\mathrm{obs}$, defined as the sample average 
of centered responses for all mice that received non-null treatments, and the optimal value, 
$V_\mathrm{opt}$, defined as the sample average across PDX lines of the maximum centered response 
across treatments, i.e., 
$V_\mathrm{opt} = m_k^{-1} \sum_{j = 1}^{m_k} \mathrm{max}_{i = 1, \ldots, P_k} R_{ijk}$. 
The observed value and optimal value can be used to define two metrics for evaluating estimated ITRs: 
proportion of optimal value, defined as 
$P_\mathrm{opt}\left(\widehat{D}^*\right) = \bar{V}\left(\widehat{D}^*\right) / V_\mathrm{opt}$, 
and ratio to observed value, defined as
$P_\mathrm{obs}\left(\widehat{D}^*\right) = \bar{V}\left(\widehat{D}^*\right) / V_\mathrm{obs}$.

\section{Results}\label{results}


All analyses were performed separately for each cancer type. We applied various combinations of the modeling decisions outlined in Section~\ref{variants}.  
Each modeling variant was applied to the feature set resulting from screening with different 
values of $L_{\mathrm{SUP}}$ (see Section~\ref{screen}), utilizing the original features or the features extracted from the DAE. We present results for BAR and defer results for TTD to Section 5 of the Supplemental Materials. 

Figure~\ref{fig:score0} illustrates the strong linear relationship between 
the mean value under each estimated ITR and the optimal value for the associated cancer type and 
treatment grouping. Note that in the majority of cases the 
optimal values do not vary within a cancer type. However, in some cases 
there is variability in the optimal value due to the fact that we select $c_1$ and $c_2$ 
separately for each analysis. Each point in Figure~\ref{fig:score0} represents 
a particular estimated ITR within a particular cancer. We also observe a 
similar relationship between the mean and observed values for each estimated ITR. 
This suggests that the marginal mean outcome differs across cancer types. 
This is not unexpected, as some cancers are known to be more sensitive 
to available treatment options (e.g., CM), and others less so (e.g., PDAC). This observation motivates our use of 
$P_\mathrm{opt}\left(\widehat{D}^{*}\right)$ and $P_\mathrm{obs}\left(\widehat{D}^{*}\right)$ to evaluate 
performance of estimated ITRs. We present results for $P_\mathrm{opt}\left(\widehat{D}^{*}\right)$ here 
and defer results for $P_\mathrm{obs}\left(\widehat{D}^{*}\right)$ to Section 4 of the Supplemental Materials.

\begin{figure}[htp]
\centering
\includegraphics{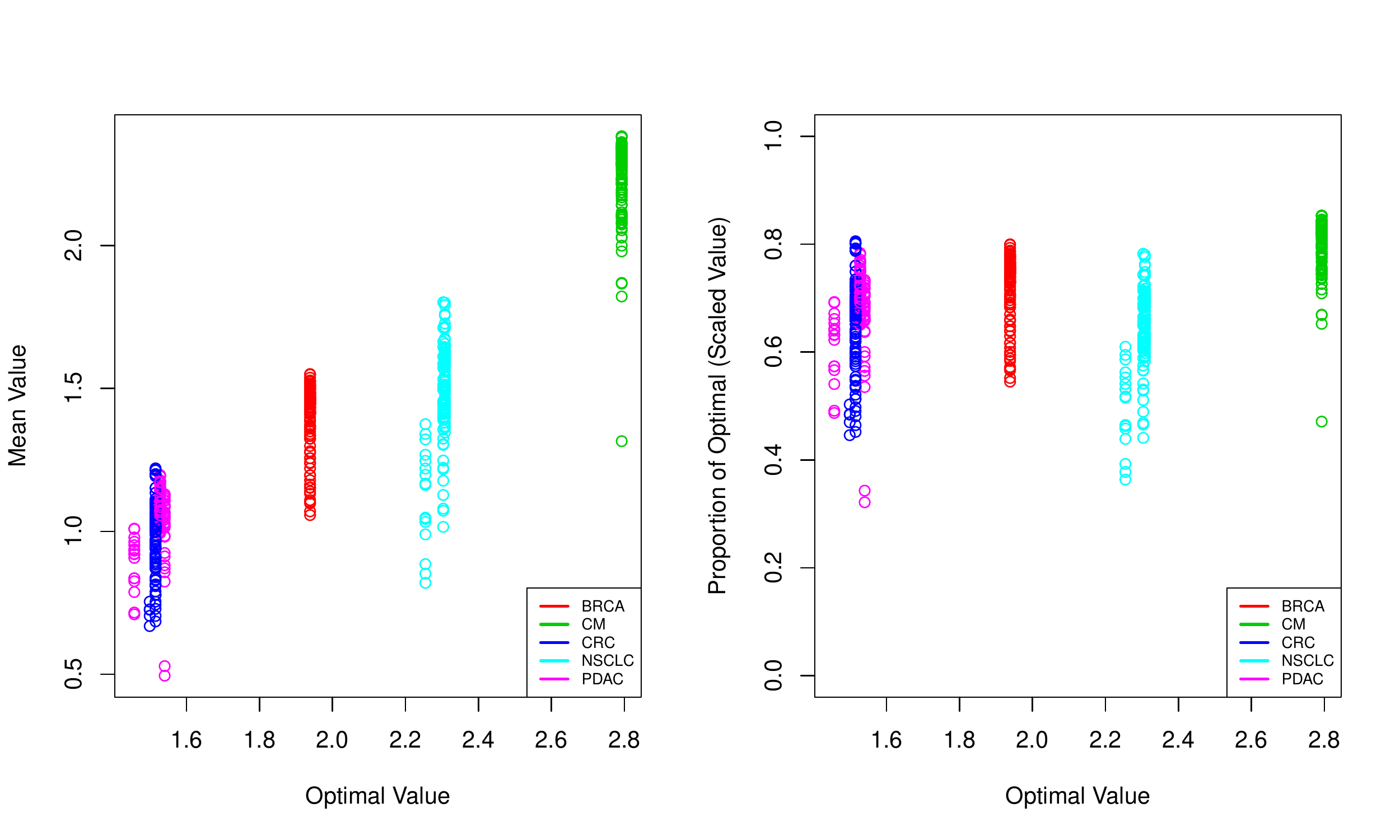}
\caption{Original (left) and scaled (right) values (for best average response) 
from all analyses, across methods, cancer types, and $L_{\text{SUP}}$. 
Optimal values for each method vary significantly by cancer type. The value 
of each estimated ITR is correlated with the optimal value across cancer types. 
We normalize the estimated value for each method by the optimal value to 
allow for comparisons between cancers. The metric, called ``proportion of optimal," 
provides a measure of how close the value of an estimated ITR is to the optimal value.}\label{fig:score0}
\end{figure}

\subsection{Relative Performance of Methods Pooling Across Conditions}

We summarize $P_\mathrm{opt}\left(\widehat{D}^{*}\right)$ for each variant of the proposed methods and the ``off-the-shelf" methods in Figure \ref{fig:scorebymethod}, pooling results across different cancer types and  values of $L_{\text{SUP}}$. From this figure, several striking trends are apparent.  For example, the application of data smoothing prior to ITR estimation boosts the overall performance for many methods, such as Q-Learning with embedded linear models (light  and dark green), and OWL methods (light and dark grey).  This pre-smoothing is less beneficial for Q-Learning methods utilizing embedded random forests (salmon and red), as the pre-smoothing itself is performed using random forests.   Q-learning methods using pseudo-values (red, $\text{QL}_{\text{2}}$) performed similarly to their $\text{QL}_\text{1}$ counterparts (pink) across conditions.   In general,  Q-learning performed better than OWL overall across  variable conditions.  In addition, using the lower dimensional features extracted from the DAE did not show significant benefit for most approaches with the exception of the OWL methods using the linear kernel,  which we will show later to be sensitive to the dimension of the feature space.  Q-Learning methods with non-linear embedded models (salmon, red) showed much better robustness to various modeling choices than linear ones, and showed to be especially helpful in OWL (dark grey).    Lastly, simpler off-the-shelf methods showed lower performance and much higher variability across conditions compared to methods utilizing the treatment tree approach. Relative to the LASSO, almost all methods utilizing the treatment tree performed better than the simpler off the shelf methods.

We also find that  utilizing a weighted combination of ITRs using the superlearner approach (dark blue) resulted in the best overall performance across conditions.  Increasing the number of ITRs included in the superlearner had the effect of boosting performance while also reducing variability in performance across conditions.  Here, the SL4 combined all four Q-Learning methods from Fig \ref{fig:scorebymethod} utilizing pre-smoothing, SL6 includes the addition of methods $\text{QL}_{\text{1,RF}}$ and $\text{QL}_{\text{2,RF}}$, SL8 adds $\text{QL}_{1}$ and $\text{QL}_{\text{2}}$, and SL16 add all Q-learning methods utilizing DAE features.   OWL methods were excluded from the superlearner due to computational time, and we expect additional performance increases through their inclusion.  This suggests that the superlearner approach can boost performance and also mitigate user uncertainty regarding the selection of various modeling approaches in cases where the optimal approach may not be clear beforehand. 

The examination of pooled overall and relative performance in Figure~\ref{fig:scorebymethod} is informative when the analyst is unsure what approach is best for their data and wants to use 
the ``safest'' approach.  For example, the results in Figure~\ref{fig:scorebymethod} can be used to determine a sequence of modeling decisions that exhibited good overall performance 
across a variety of conditions and achieved low variability across conditions. 
For example, $\text{QL}_{\text{1, smoothed}}$, or Q-learning utilizing smoothed 
outcomes and observed-values, achieved good performance on average and low variability in performance across conditions. Alternatively, one may elect to combine ITRs from multiple methods using the superlearner approach to avoid choosing an individual method, at the cost of additional computational burden.  In addition, our results suggest the following general conclusions: linear methods may benefit from the use of data smoothing prior to ITR estimation, further dimension reduction of the predictor space via DAEs is beneficial for OWL methods using the linear kernel, and nonlinear embedded models are more robust to various modeling choices than linear embedded models. 

\begin{figure}[htp]
\centering
\includegraphics{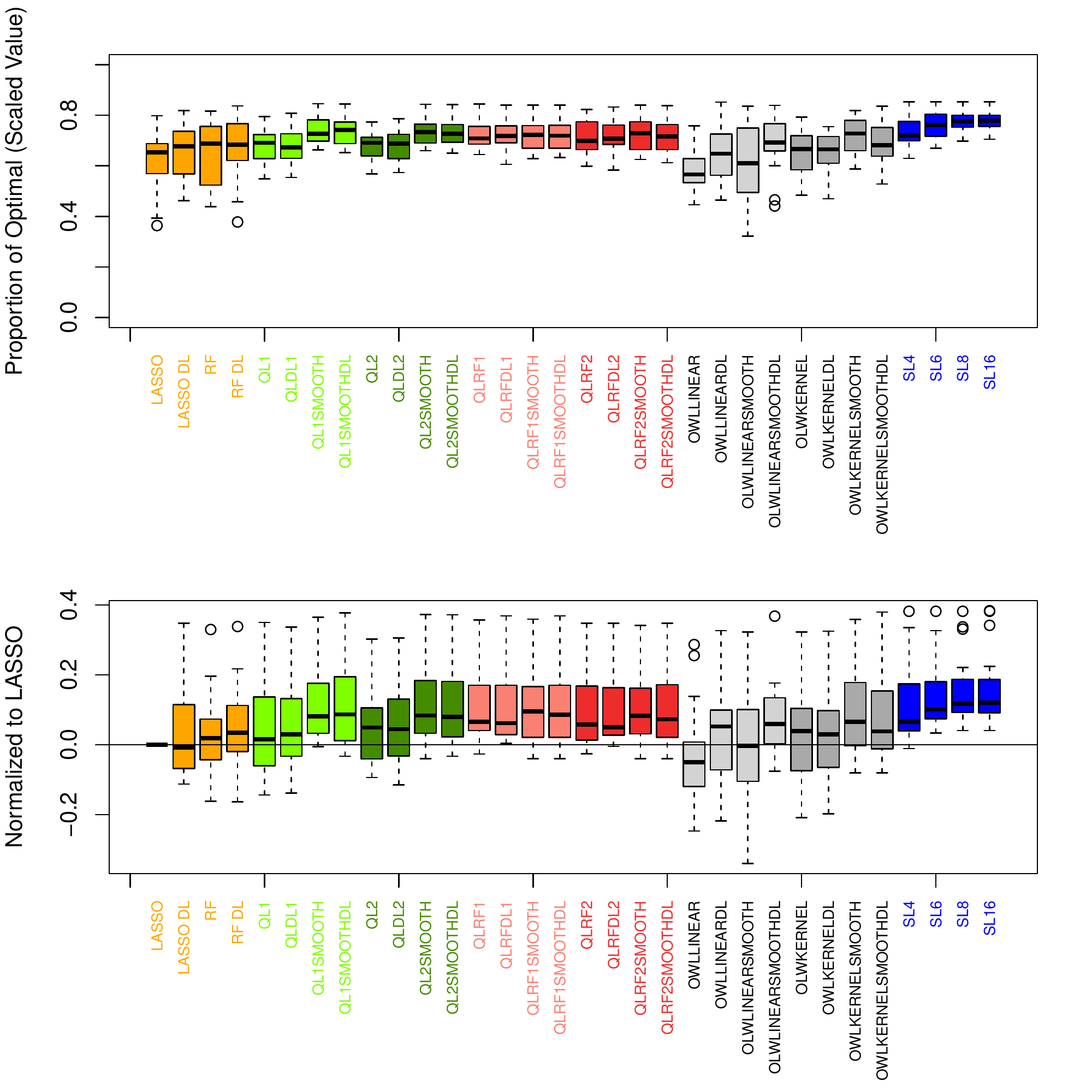}
\caption{Overall performance across methods, pooled over cancer types and number of features (top).   $P_{opt}(\widehat{D}^{*})$ for each method is also normalized to the LASSO in each condition to highlight the relative performance of each approach  (bottom).  This relative measure was constructed by subtracting the $P_{opt}(\widehat{D}^{*})$ pertaining to the LASSO from that of the other methods within each combination of cancer type and $L_{\text{SUP}}$ value.  }\label{fig:scorebymethod}
\end{figure}

\subsection{Impact of $L_{\mathrm{SUP}}$ on Performance}

We now examine the relative performance of our estimated ITRs by $L_{\mathrm{SUP}}$, 
the number of genes used in estimation. Figure~\ref{fig:scoreplsup} contains boxplots of $P_\mathrm{opt}\left(\widehat{D}^{*}\right)$ 
pooled over cancer types and stratified by estimation method and $L_{\mathrm{SUP}}$ (Supplementary Table 1).  Prior to pooling we adjust $P_\mathrm{opt}\left(\widehat{D}^{*}\right)$ in each cancer type and method by the performance at  $P_\mathrm{opt}\left(\widehat{D}^{*}\right) = 50$ to more clearly delineate changes with respect to dimension.   Most methods did not show strong trends in performance 
across $L_{\mathrm{SUP}}$, with the exception of the
$\text{OWL}_{\text{linear}}$ class of methods (light grey). In this class, $P_\mathrm{opt}\left(\widehat{D}^{*}\right)$
decreases with increasing numbers of genes. However, the same trend did not appear when utilizing 
the Gaussian kernel. Slight downward trends were also observed 
for Q-learning with an embedded linear model (light and dark green). 
These observations together suggest that methods in which a linear 
decision rule is estimated at each step of the treatment tree may be sensitive to the dimension 
of the feature space. 

The optimal set of features for ITR estimation (indexed by $L_{\mathrm{SUP}}$)  is the set of all those features for 
which at least one of the $h_j$, $j = 0, \ldots, J$ in model~(\ref{model}) is not constant. While 
the optimal set of features is unknown, the results in Figure~\ref{fig:scoreplsup} 
indicate that the proposed treatment tree-based ITR estimation method performs well 
regardless of the number of features selected, with the exception of linear OWL and 
Q-learning with an embedded linear model. These results lead to the conclusion that 
a smaller set of features post-supervised screening may be optimal for methods with an embedded linear model, and 
an embedded nonlinear model can be used to provide robustness against selecting the ``wrong'' set of features.  

\begin{figure}[htp]
\centering
\includegraphics{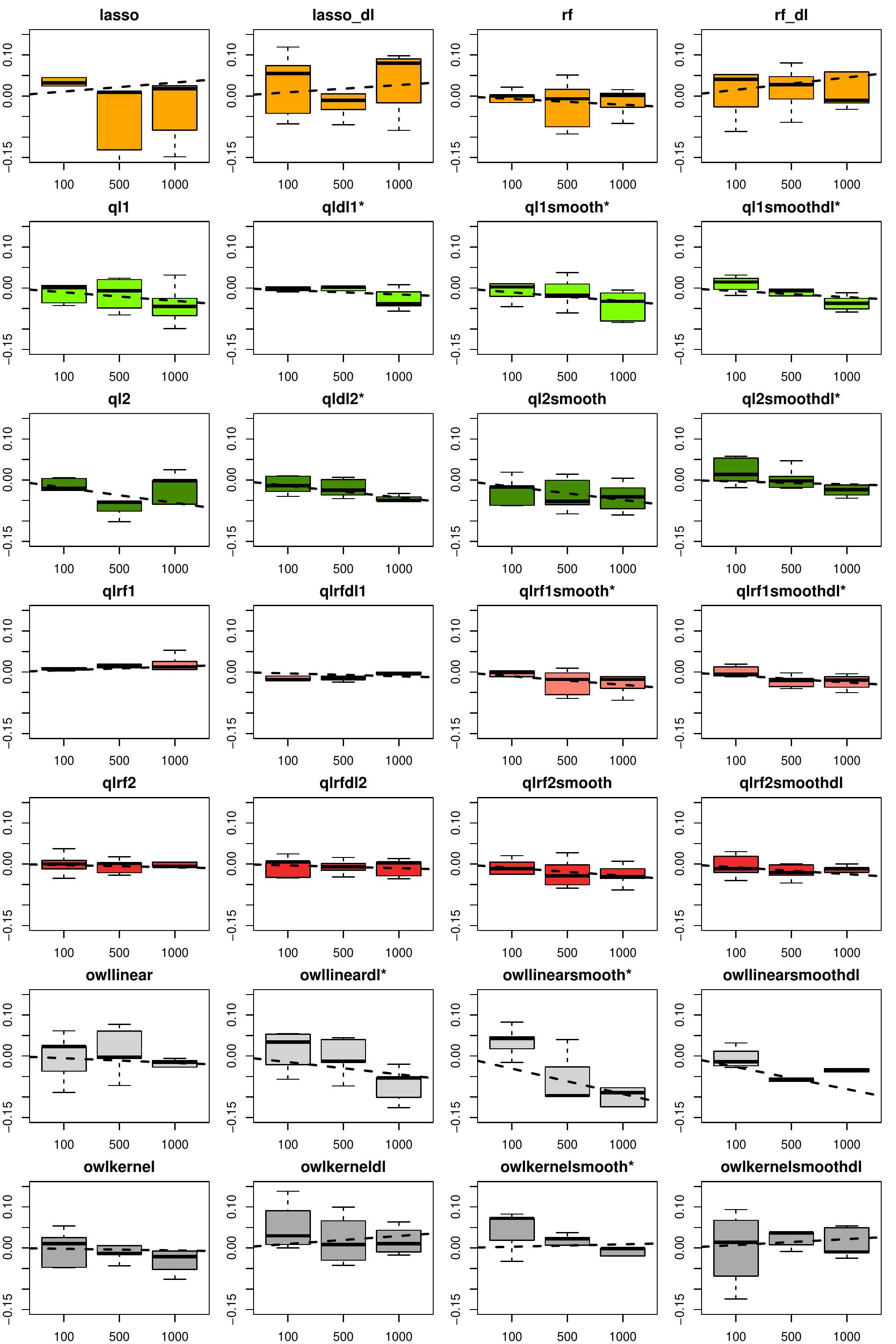}
\caption{Overall trends in performance for each method across 
$L_{\mathrm{SUP}}$, aggregated over cancer types.}\label{fig:scoreplsup}
\end{figure}

\subsection{Overall Performance by Cancer}

The performance resulting from specific modeling decisions varies across cancer types. 
Figure~\ref{fig:bycancer} contains boxplots of $P_{\mathrm{opt}}\left(\widehat{D}^{*}\right)$ 
for the best performing set of modeling decisions within the classes defined 
by the colors in Figure~\ref{fig:scorebymethod}. 
To select the best performing variant in each class, we chose the one with the largest 
$P_\mathrm{opt}\left(\widehat{D}^{*}\right)$ averaged across values of $L_{\mathrm{SUP}}$ 
within that particular class. The boxplots in Figure~\ref{fig:bycancer} contain 
$P_\mathrm{opt}\left(\widehat{D}^{*}\right)$ over values of $L_{\mathrm{SUP}}$. 
We show the full set of results corresponding 
to all methods and cancers in the Section 3 of the Supplemental Materials. 
Within certain cancers, such as breast cancer (BRCA), specific modeling decisions do not 
result in large differences in performance. This may be due in part to the 
fact that a small number of treatments 
appear to work well uniformly across samples in BRCA (Supplementary Figure 8). 
In contrast, for pancreatic cancer (PDAC) and non-small cell lung cancer (NSCLC), 
greater heterogeneity in response exists across treatments 
(Supplementary Figures 9 and 10). In addition, we find that in almost all cancer types, the superlearner tended to perform better than or similar to all other classes of methods, suggesting its use when it is unclear which individual method may be optimal for a particular dataset. 

Table~\ref{tab:besbycancer} lists the best overall set of modeling decisions 
for each cancer type, along with the genomic features that were most important 
for selecting treatments using the best performing estimated ITR. 
For cancers where the best performing ITR resulted from the DAE predictors or OWL methods using the Gaussian kernel, 
it is difficult to determine which genes were the most important. For these cases, 
we selected important features using the second-best  ITR 
(Reference Method in Table~\ref{tab:besbycancer}).  For all cancer types, the best performing ITR resulted from the tree-based approach rather than an ``off-the-shelf" method.  In addition, despite their sensitivity to dimension, the $\text{OWL}_{\text{linear}}$ class of methods were represented as the best ITR in two out of five cancers.  This suggests that as long as the correct feature dimension is selected beforehand, OWL method can perform well relative to other methods.  In practice however, the optimal dimension is difficult to ascertain unless one evaluates multiple candidate feature sets, as we have in this manuscript. 

The treatments most frequently recommended by the best performing ITR for each cancer, along with the corresponding values of $c_1$ and $c_2$, are given in Supplementary Table 6. The treatment tree for the optimal 
ITR varied in the values of $c_2$ across cancer types, suggesting variability in 
the amount of response heterogeneity across cancer types.  For example, in CM, where response tended to be strong overall (Figure~\ref{fig:score0}),  $c_2=9$, suggesting that relatively more treatments had similar response profiles across PDX lines.
The selected value of $c_1$ for the best performing ITR was low for each cancer type, 
indicating that only a small number of treatments were found to be effectively the same as ``untreated" based on the hierarchical clustering (Supplementary Table 6). 
For ``off-the-shelf" methods, $c_1 = 0$ and $c_2 = P_k$ by design since no 
grouping of treatments was performed.  


We list the average (unconditional) response for each of the $P_k$ 
treatments within cancer type in Supplementary Table 8, calculated as the sample 
average response across mice treated with each of the $P_k$ treatments across 
PDX lines. In BRCA, the treatment with the larger mean response was also the most 
recommended treatment (LEE011 + everolimus). In PDAC, however, there is less variability 
in average response across treatments, and the best performing ITR recommends 
BKM120 + binimetinib to 18 PDX lines and abraxane + gemcitabine to 18 PDX lines 
(see Supplementary Table 7). 

\begin{figure}[htp]
\centering
\includegraphics{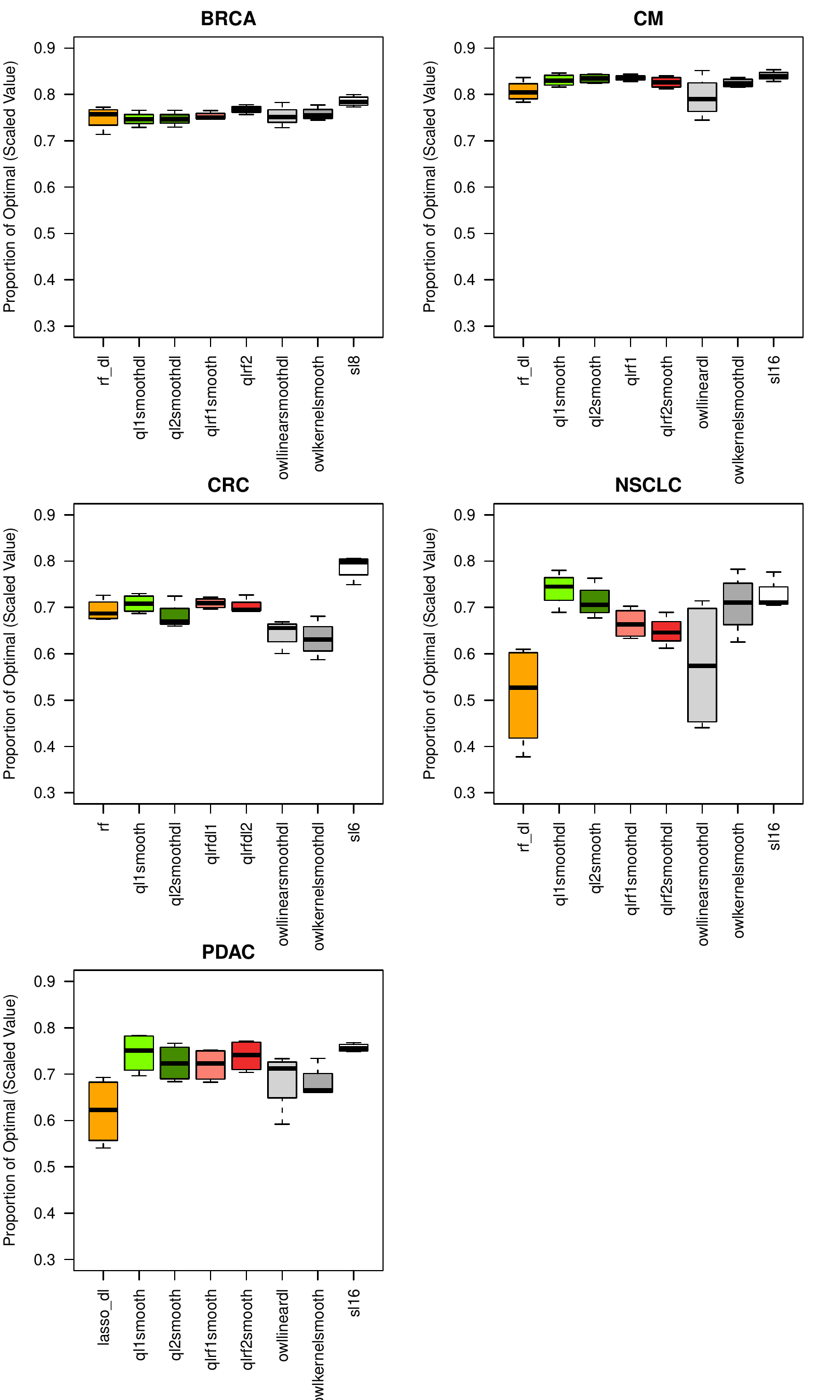}
\caption{Performance of the best estimated ITR in each class across cancer types.}\label{fig:bycancer}
\end{figure}
 
\begin{table}[ht]
\centering
\begingroup\scriptsize
\begin{tabular}{lllll}
  \hline
Cancer & Method & Lsup & Reference Method & Top 5 Predictors \\ 
  \hline
BRCA & owllinearsmoothdl & 50 & qlrf2 & COL1A1.rna,CTCFL.rna,FMNL3.rna,SRPX.rna,HMCN1.cn \\ 
  CM & owllineardl & 100 & ql1smooth & ETV7.rna,DPYSL3.rna,LEPREL1.rna,GABRE.cn,ATP2B1.rna \\ 
  CRC & ql1smooth & 100 & ql1smooth & FGG.rna,ALPK1.rna,WDR27.cn,DIDO1.mut,C10orf26.rna \\ 
  NSCLC & owlkernelsmooth & 100 & ql2smooth & POFUT2.rna,TNNI3.cn,TUBG1.cn,NAT8L.cn,PTPRE.cn \\ 
  PDAC & ql1smooth & 100 & ql1smooth & CTH.rna,DUSP4.cn,TPP2.rna,ACVR1B.rna,ZNF264.cn \\ 
   \hline
\end{tabular}
\endgroup
\caption{Best performing method and number of predictors for each cancer.  Top 5 most important predictors are listed.  Predictors pertaining to the next best performing method were provided if the top performing method utilized deep learning (Reference Method).  Predictors ending in .rna are from the gene expression data, .cn from the copy number data, and .mut from the mutation dataset.} 
\label{tab:besbycancer}
\end{table}

\subsection{ITR Performance When Limiting Features to a Single Genomic Platform} \label{sec:limit}

The three genomic platforms utilized in this study 
provide a wealth of information for ITR estimation and biomarker discovery. 
However, the high dimension of the feature space provides practical and computational 
challenges. Predictors from the same gene may be 
correlated (for example, gene expression and gene copy number) and may therefore 
be redundant. Evaluating three genomic platforms for each PDX line's original tumor 
increases cost and amount of tumor tissue required. For these reasons, we also 
explored the performance of ITRs estimated using only the RNA-seq gene expression 
platform, a common genomic assay performed by biomedical researchers. We repeated 
the same process described in Section~\ref{methods}, but using only features resulting 
from RNA-seq. 

The overall conclusions are largely similar to those in Figure \ref{fig:scorebymethod} 
(see Supplementary Figure 11). When we compare the difference in $P_\mathrm{opt}\left(\widehat{D}^{*}\right)$ 
between the ITR estimated using the full feature set and the ITR estimated using 
RNA-seq only, we find that most methods perform similarly using only RNA-seq data 
(Supplementary Figure 12). Q-learning and OWL methods with embedded linear models did slightly better than corresponding ITRs trained on all three platforms, which is likely related to the relative sensitivity of these methods to the dimensions of the feature space (smaller in the RNA-seq only analysis).  Otherwise, most methods evaluated in RNA-seq only performed similarly to ITRs trained on all three platforms.  Due to the smaller feature space in the RNA-seq only analysis we did not elect to run the deep learning variants of each method.  Overall, these results suggest that utilizing a single genomic platform 
may be a more cost-effective option that will result in estimated ITRs with comparable performance. 

\subsection{Top Genomic Features in NSCLC}

Lung cancer is one of the leading causes of cancer related deaths in the 
United States, and NSCLC accounts for the majority 
of clinical cases of lung cancer \citep{ettinger2010non}. Therapeutic agents used to 
treat NSCLC include paclitaxel, which interferes with cellular 
microtubular dynamics and cellular division through the targeting of 
tubulin \citep{wise2000gamma}, and cetuximab, which 
targets the epidermal growth factor receptor (EGFR) \citep{pirker2009cetuximab}. 
Since the mechanism of action differs between these two treatments, 
it is not surprising that we observed significant heterogeneity in response between these 
two treatments in this study (Supplementary Figure 10). 
Here, we examine the relationship between response and the 
genomic features found to be the most important for making decisions using 
the best performing estimated ITR. 

The best performing ITR for NSCLC resulted from $\text{OWL}_{\text{kernel,smoothed}}$ 
(see Table~\ref{tab:besbycancer}).  Given that that gaussian kernel for OWL does not allow direct interpretation of its predictors, we utilize the reference method $\text{QL}_{\text{2,smoothed}}$ in this cancer to examine the role of each selected predictor with respect to response.    We determined the most important 
genomic features for this ITR using the following approach. For each of the $c_2$ splits in the 
associated treatment tree, we computed the absolute value of the regression coefficient for each feature in the model fit at a given node and retained the feature with the largest magnitude value at each node. 
Cross-validation selected $c_2 = 13$; therefore, we retained 
a set of 13 top features, where 10 of these were unique (Figure~\ref{fig:NSCLC}). We then calculated Spearman's rank correlation coefficient between each selected feature and 
response, separately within each non-null treatment. This resulted in the matrix of 
feature-treatment pairwise correlation coefficients seen in Figure~\ref{fig:NSCLC}. 
We then clustered the columns of this matrix, pertaining to the seven selected genomic 
features, using Euclidean distance between the vector of correlation 
coefficients in each column. We also clustered the rows of this matrix, pertaining to 
the non-null treatments, using the tree structure that was previously constructed for 
$\text{QL}_{\text{2,smoothed}}$, rather than the correlations. 

The treatment with one of the strongest correlation to Tubulin Gamma 1 (TUBG1) 
copy number is the therapy paclitaxel, suggesting 
that a higher copy number of TUBG1 may potentiate response in patients 
being targeted with agents such as paclitaxel. 
This is notable because paclitaxel directly targets tubulin \citep{kumar1981taxol}, 
of which TUBG1 plays a major role. 
Furthermore, cetuximab is the only treatment that is anticorrelated 
with TUBG1 copy number. This unique relationship with TUBG1 is 
reflected in the treatment tree, as cetuximab is the only member of a branch 
furthest away from all other non-null treatments. Cetuximab also exhibits a 
unique mechanism of action, as the only monoclonal antibody 
EGFR inhibitor in our data set. 
The relationship between TUBG1 copy number and response to cetuximab and 
paclitaxel is displayed in Figure~\ref{fig:NSCLC}. These results indicate that 
that the observed correlations between treatment response and the top features 
determined from $\text{QL}_{\text{2,smoothed}}$ reflect the role these features play 
as important variables in the best performing ITR. 

\begin{figure}[htp]
\begin{minipage}{1.0\textwidth}
\centering
\includegraphics{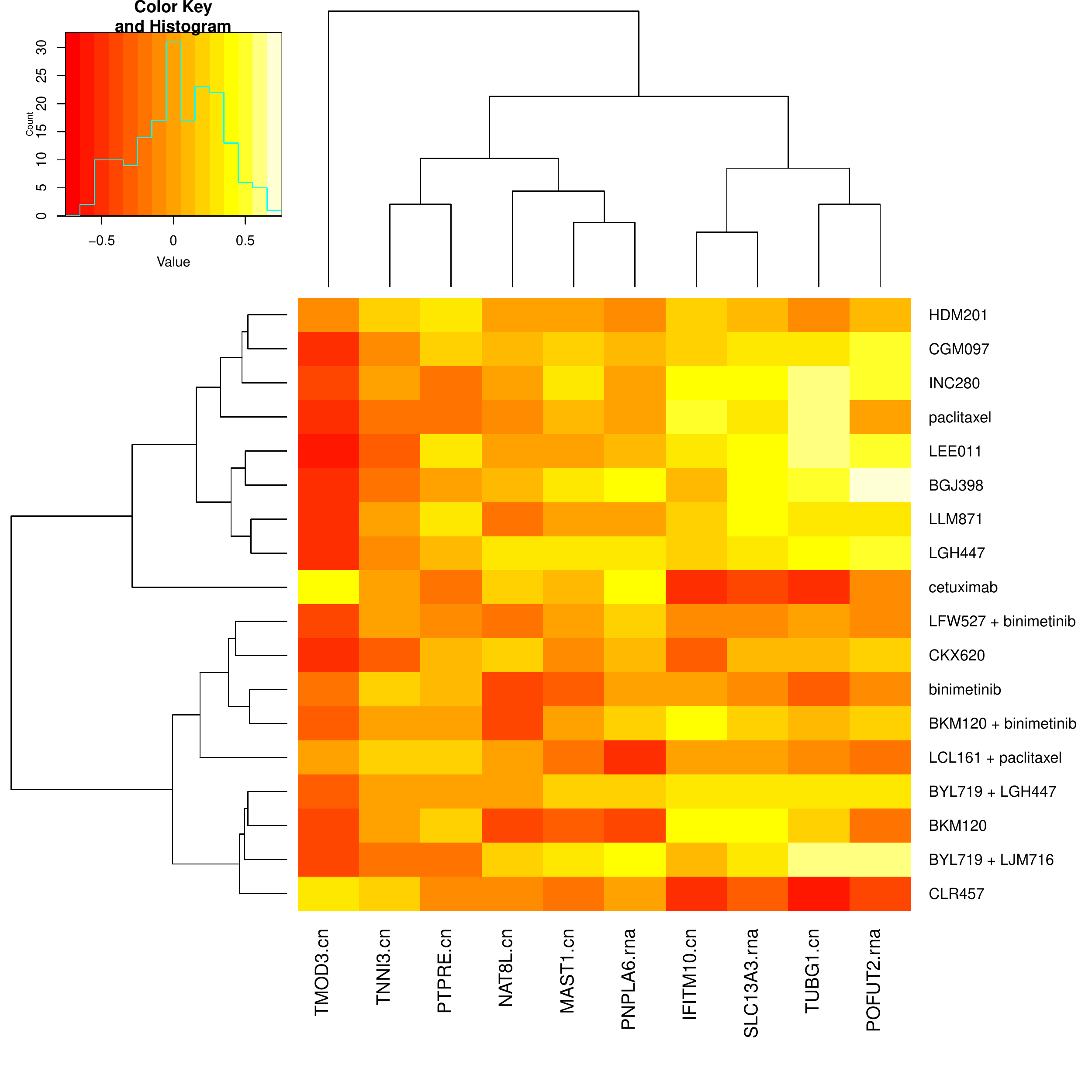}
\end{minipage}
\begin{minipage}{\textwidth}
\centering
\includegraphics{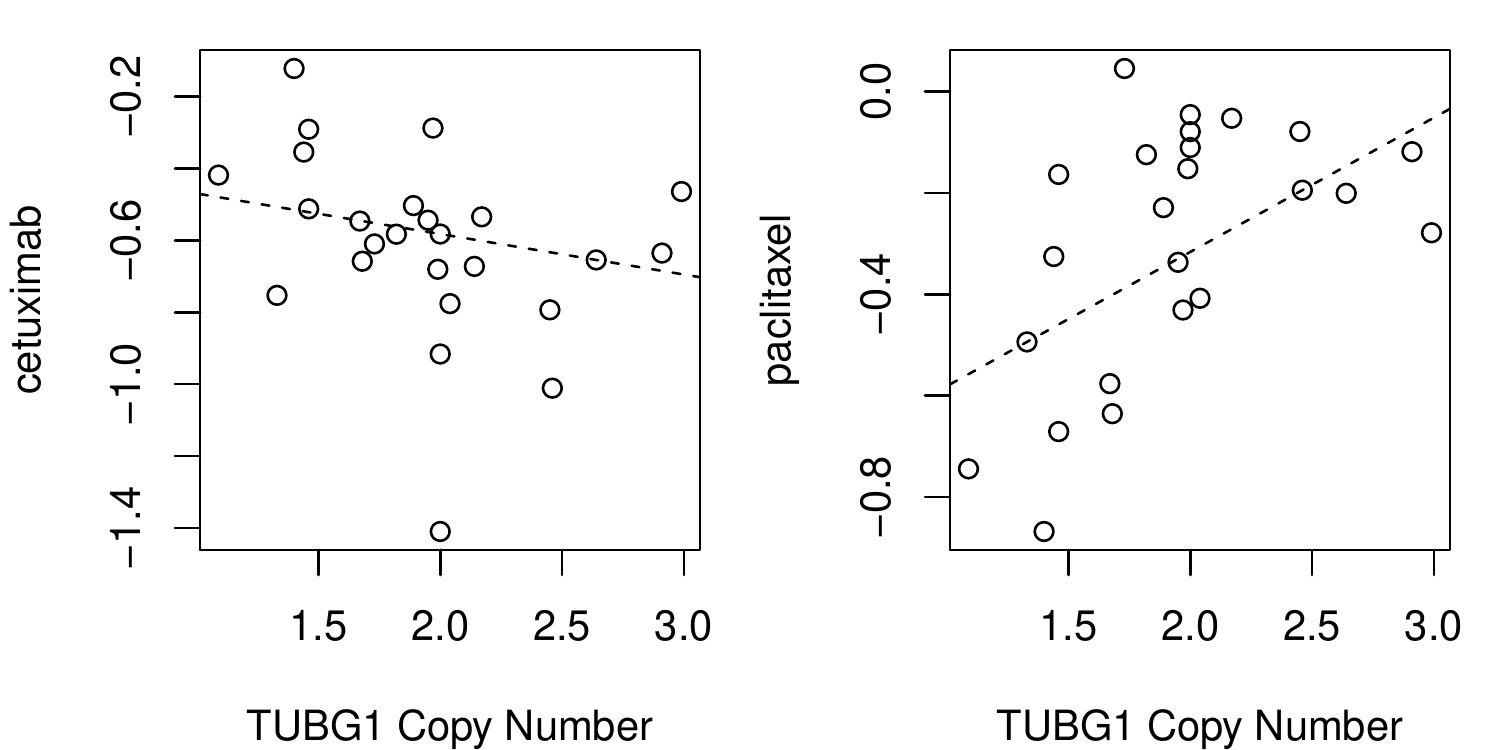}
\end{minipage}\caption{Top genomic features selected by $\text{QL}_{\text{2,smoothed}}$
in NSCLC. Cells are colored by the magnitude of their Spearman's 
correlation between response to treatment (rows) and expression of top features (columns). 
Genomic features were clustered by their vector of Spearman's correlation to each treatment 
(using Euclidean distance), whereas treatments were grouped based on the treatment 
tree constructed for $QL_{\text{2,smoothed}}$. The patterns of treatment-feature 
correlations tended to coincide with the predetermined treatment groupings.}\label{fig:NSCLC}
\end{figure}

\section{Discussion} \label{discussion}

In this paper, we introduced several approaches to ITR estimation using 
PDX data. The unique structure of a PDX study, where multiple treatments are 
applied to samples from the same human tumor implanted into mice, naturally lends 
itself to precision medicine. The substantially improved precision that results 
from the PDX structure may result in better performing ITRs. However, PDX data 
also pose a number of challenges, including a large number of 
unordered treatments and a high-dimensional feature space. These factors make it 
difficult to nonparametrically model the conditional mean of the response. 
The method we propose involves of sequence of steps that alleviate these difficulties. 
Our method involves screening the covariates to find a lower dimensional feature space, 
constructing a treatment tree that allows for recasting ITR estimation as a sequence of 
binary decisions, and estimating a decision rule at each split to select the arm 
that contains the optimal treatment. Because we aim to select the arm that contains the 
optimal treatment at each split rather than the arm with the largest average 
response, our estimation technique utilizes the maximum response downstream of each arm for 
each line. Thus, our estimation technique incorporates the unique structure of the PDX data 
by directly using the multiple responses observed per PDX line. We've shown that the 
method we propose not only produces high-quality ITRs, but also identifies genes that 
are known to be associated with response to treatment (e.g., the TUBG1 gene shown in 
Figure~\ref{fig:NSCLC}). 

The methods we propose requires making a number of modeling decisions at various stages of 
the pipeline, including selecting the dimension of the feature space, choosing embedded models, 
and selecting tuning parameters, among others. We demonstrated various combinations of these 
modeling decisions in our analyses. While certain variants performed better than others for 
certain cancer types, the treatment tree-based approach outperformed ``off-the-shelf" methods 
overall. Reducing noise by using random forest-predicted outcomes (smoothing) 
improved performance of the estimated ITRs in most settings, and basing ITRs on 
DAE-extracted features improved performance in the presence of a linear embedded model. 
In particular, Q-learning with smoothed outcomes and pseudo-values performed well 
across all settings and is a good first choice for estimating ITRs from PDX data.  
We recognize that the models, tuning parameters, and implementation discussed here, 
despite our best efforts, may not be optimal. The method studied here 
could potentially be improved through careful tuning. Studying the proposed 
method further, including through extensive simulation experiments, could yield 
more concrete recommendations as to which embedded models perform best.  Our implementation of a superlearner consisting of multiple estimated ITRs was shown to improve performance above individual ITRs.  This approach helps mitigate user uncertainty regarding the best choice of ITR estimation approach or modeling choices for a given problem, where one may combine the results from multiple ITRs using the superlearner for improved performance. 

Many of the modeling decisions in this paper were made with computational intensity in mind. 
In our analyses, we utilized two large computing clusters at the University of 
North Carolina at Chapel Hill, the Killdevil cluster with 9500 computing cores and 
the Longleaf cluster with 3600 computing cores. Future improvements on the proposed method 
must be implemented in a computationally efficient way. We also recognize that 
there are many methods we could have used that may have performed better, and that 
different methodological approaches entirely may be able to improve 
upon these results. Nevertheless, the analyses in this paper demonstrate the potential for 
using PDX studies to inform precision medicine. 

The assumption that ITRs estimated from PDX data are applicable to humans is crucial 
to this work. This assumption is based on decades of biomedical research 
on generalizing PDX results to humans. A major conclusion of \cite{gao2015high} 
was that the responses observed in their PDX lines correlated 
with the responses observed in the human patients from which the tumors 
were taken. Many prior studies have shown that PDX models show stronger correlation 
with human response compared to traditional cell line models 
\citep{whittle2015patient, scott2014patient, rosfjord2014advances}. Prior work has also 
shown that the tumor microenvironment, consisting of stromal cells and other tissue, 
may impact tumor activity and response to treatment. In PDX models, the microenvironment 
surrounding the implanted human tumor is non-human. However, several recent studies 
have suggested that tumor recruitment of mouse stroma in PDX models mirrors that seen in 
humans \citep{roife2016ex, wang2017breast} and may show similarity in treatment 
response when specifically targeted. 

A key next step for this research will be to plan and conduct validation studies in 
humans to determine if the biomarkers and ITRs discovered here can be used to 
improve outcomes in human patients. We note that ITRs based on only one genomic 
platform (RNA-seq) resulted in comparable performance to ITRs based on three genomic platforms. 
Since independent validation 
studies will be more expensive if more genomic platforms are needed, validating the 
ITRs based only on RNA-seq data using a study of human cancer patients would be a 
low-cost initial step toward validating the results in this paper. 
One advantage of the proposed method is that, in some settings, we could apply the estimated 
ITRs to external studies that included only a subset of the treatments applied in this study. 
Given a sequence of decision rules (pertaining to each step of the treatment tree), 
one could start at the lowest node that contains all of the treatments of interest 
and follow the tree to arrive at a recommended treatment, rather than starting from the top of the tree. 
We also note that, while our results allow for comparing the performance of estimated ITRs 
across different cancer types and modeling strategies, the data utilized for this paper do not 
allow for comparing the performance of ITRs estimated from PDX studies to ITRs 
estimated from human trials. Another important step forward for validating these results will 
be to compare the treatment rules discovered here to those estimated from human trials to 
determine the value of PDX studies for precision medicine. 

The design of a PDX study plays a key role in ITR estimation, and designing 
high-quality PDX studies is another key next step for this research. Our results suggest that 
the observed responses in PDX studies are noisy. Having replicates within PDX line, i.e., 
multiple mice per line assigned to the same treatment, may improve the performance of the estimated 
ITRs. The design could also be improved by having a larger number of distinct tumor 
lines to ensure sufficient representation of cancer heterogeneity across a diverse spectrum 
of cancer patients. 

While our research has, in some ways, raised more questions than it has answered, 
we feel that the treatment rules, biomarkers, and other results we have discovered 
here are interesting in their own right and merit further research, including validation 
studies. Future research in this area will allow us to make measurable advances in 
applying PDX studies in precision medicine and translating the results into clinical 
practice. Plans for such future research is underway.

\appendix

\makeatletter   
 \renewcommand{\@seccntformat}[1]{APPENDIX~{\csname the#1\endcsname}.\hspace*{1em}}
 \makeatother

\bibliographystyle{ECA_jasa}
\bibliography{example}

@article{hidalgo2014patient,
  title={Patient-derived Xenograft Models: {A}n Emerging Platform for Translational Cancer Research},
  author={Hidalgo, Manuel and Amant, Frederic and Biankin, Andrew V and Budinsk{\'a}, Eva and Byrne, Annette T and Caldas, Carlos and Clarke, Robert B and de Jong, Steven and Jonkers, Jos and M{\ae}landsmo, Gunhild Mari and others},
  journal={Cancer Discovery},
  volume={4},
  number={9},
  pages={998--1013},
  year={2014},
  publisher={AACR}
}

@article{gao2015high,
  title={High-throughput Screening Using Patient-derived Tumor Xenografts to Predict Clinical Trial Drug Response},
  author={Gao, Hui and Korn, Joshua M and Ferretti, St{\'e}phane and Monahan, John E and Wang, Youzhen and Singh, Mallika and Zhang, Chao and Schnell, Christian and Yang, Guizhi and Zhang, Yun and others},
  journal={Nature Medicine},
  volume={21},
  number={11},
  pages={1318--1325},
  year={2015},
  publisher={Nature Publishing Group}
}

@article{siolas2013patient,
  title={Patient-derived Tumor Xenografts: {T}ransforming Clinical Samples Into Mouse Models},
  author={Siolas, Despina and Hannon, Gregory J},
  journal={Cancer Research},
  volume={73},
  number={17},
  pages={5315--5319},
  year={2013},
  publisher={AACR}
}

@article{zhao2012estimating,
  title={Estimating Individualized Treatment Rules Using Outcome Weighted Learning},
  author={Zhao, Yingqi and Zeng, Donglin and Rush, A John and Kosorok, Michael R},
  journal={Journal of the American Statistical Association},
  volume={107},
  number={499},
  pages={1106--1118},
  year={2012},
  publisher={Taylor \& Francis Group}
}

@article{polyak2011heterogeneity,
  title={Heterogeneity in Breast Cancer},
  author={Polyak, Kornelia},
  journal={The Journal of Clinical Investigation},
  volume={121},
  number={10},
  pages={3786--3788},
  year={2011},
  publisher={Am Soc Clin Investig}
}

@article{tentler2012patient,
  title={Patient-derived Tumour Xenografts as Models for Oncology Drug Development},
  author={Tentler, John J and Tan, Aik Choon and Weekes, Colin D and Jimeno, Antonio and Leong, Stephen and Pitts, Todd M and Arcaroli, John J and Messersmith, Wells A and Eckhardt, S Gail},
  journal={Nature Reviews Clinical Oncology},
  volume={9},
  number={6},
  pages={338--350},
  year={2012},
  publisher={Nature Publishing Group}
}

@article{sargent2005clinical,
  title={Clinical Trial Designs for Predictive Marker Validation in Cancer Treatment Trials},
  author={Sargent, Daniel J and Conley, Barbara A and Allegra, Carmen and Collette, Laurence},
  journal={Journal of Clinical Oncology},
  volume={23},
  number={9},
  pages={2020--2027},
  year={2005},
  publisher={American Society of Clinical Oncology}
}

@article{metzger2012dissecting,
  title={Dissecting the Heterogeneity of Triple-negative Breast Cancer},
  author={Metzger-Filho, Otto and Tutt, Andrew and de Azambuja, Evandro and Saini, Kamal S and Viale, Giuseppe and Loi, Sherene and Bradbury, Ian and Bliss, Judith M and Azim Jr, Hatem A and Ellis, Paul and others},
  journal={Journal of Clinical Oncology},
  volume={30},
  number={15},
  pages={1879--1887},
  year={2012},
  publisher={American Society of Clinical Oncology}
}

@article{chen2014non,
  title={Non-small-cell Lung Cancers: {A} Heterogeneous Set of Diseases},
  author={Chen, Zhao and Fillmore, Christine M and Hammerman, Peter S and Kim, Carla F and Wong, Kwok-Kin},
  journal={Nature Reviews Cancer},
  volume={14},
  number={8},
  pages={535--546},
  year={2014},
  publisher={Nature Research}
}

@article{parker2009supervised,
  title={Supervised Risk Predictor of Breast Cancer Based on Intrinsic Subtypes},
  author={Parker, Joel S and Mullins, Michael and Cheang, Maggie CU and Leung, Samuel and Voduc, David and Vickery, Tammi and Davies, Sherri and Fauron, Christiane and He, Xiaping and Hu, Zhiyuan and others},
  journal={Journal of Clinical Oncology},
  volume={27},
  number={8},
  pages={1160--1167},
  year={2009},
  publisher={American Society of Clinical Oncology}
}

@article{zhao2009reinforcement,
  title={Reinforcement Learning Design for Cancer Clinical Trials},
  author={Zhao, Yufan and Kosorok, Michael R and Zeng, Donglin},
  journal={Statistics in Medicine},
  volume={28},
  number={26},
  pages={3294--3315},
  year={2009},
  publisher={Wiley Online Library}
}

@article{zhou_residual_2015,
  title={Residual Weighted Learning for Estimating Individualized Treatment Rules},
  author={Zhou, Xin and Mayer-Hamblett, Nicole and Khan, Umer and Kosorok, Michael R},
  journal={Journal of the American Statistical Association},
  volume={112},
  number={517},
  pages={169--187},
  year={2017},
  publisher={Taylor \& Francis}
}

@article{liu_robust_2016,
	title = {Robust Hybrid Learning for Estimating Personalized Dynamic Treatment Regimens},
	url = {http://arxiv.org/abs/1611.02314},
	journal = {arXiv:1611.02314 [stat]},
	author = {Liu, Ying and Wang, Yuanjia and Kosorok, Michael R. and Zhao, Yingqi and Zeng, Donglin},
	month = nov,
	year = {2016},
	note = {arXiv: 1611.02314},
}

@article{chen_estimating_2017,
	title = {Estimating Individualized Treatment Rules for Ordinal Treatments},
	url = {http://arxiv.org/abs/1702.04755},
	urldate = {2017-02-21},
	journal = {arXiv:1702.04755 [stat]},
	author = {Chen, Jingxiang and Fu, Haoda and He, Xuanyao and Kosorok, Michael R. and Liu, Yufeng},
	month = feb,
	year = {2017},
	note = {arXiv: 1702.04755}
}

@article{schulte_q_2014,
	title = {Q- and {A}-learning Methods for Estimating Optimal Dynamic Treatment Regimes},
	volume = {29},
	issn = {0883-4237},
	doi = {10.1214/13-STS450},
	language = {eng},
	number = {4},
	journal = {Statistical Science: A Review Journal of the Institute of Mathematical Statistics},
	author = {Schulte, Phillip J. and Tsiatis, Anastasios A. and Laber, Eric B. and Davidian, Marie},
	month = nov,
	year = {2014},
	pmid = {25620840},
	pmcid = {PMC4300556},
	pages = {640--661}
}

@article{zhao_reinforcement_2011,
	title = {Reinforcement Learning Strategies for Clinical Trials in Nonsmall Cell Lung Cancer},
	volume = {67},
	issn = {1541-0420},
	doi = {10.1111/j.1541-0420.2011.01572.x},
	language = {en},
	number = {4},
	urldate = {2017-02-21},
	journal = {Biometrics},
	author = {Zhao, Yufan and Zeng, Donglin and Socinski, Mark A. and Kosorok, Michael R.},
	month = dec,
	year = {2011},
	pages = {1422--1433}
}

@article{szekely2009brownian,
  title={Brownian Distance Covariance},
  author={Sz{\'e}kely, G{\'a}bor J and Rizzo, Maria L and others},
  journal={The Annals of Applied Statistics},
  volume={3},
  number={4},
  pages={1236--1265},
  year={2009},
  publisher={Institute of Mathematical Statistics}
}

@article{wang2017adnn,
  title={Sufficient Markov Decision Processes},
  author={Wang, Longshaokan and Laber, Eric},
  journal={Submitted},
  year={2017}
}

@article{vincent_stacked_2010,
	title = {Stacked Denoising Autoencoders: {L}earning Useful Representations in a Deep Network with a Local Denoising Criterion},
	volume = {11},
	issn = {ISSN 1533-7928},
	shorttitle = {Stacked {Denoising} {Autoencoders}},
	number = {Dec},
	urldate = {2017-02-25},
	journal = {Journal of Machine Learning Research},
	author = {Vincent, Pascal and Larochelle, Hugo and Lajoie, Isabelle and Bengio, Yoshua and Manzagol, Pierre-Antoine},
	year = {2010},
	pages = {3371--3408}
}

@article{kumar1981taxol,
  title={Taxol-induced Polymerization of Purified Tubulin. {M}echanism of Action.},
  author={Kumar, Nirbhay},
  journal={Journal of Biological Chemistry},
  volume={256},
  number={20},
  pages={10435--10441},
  year={1981},
  publisher={ASBMB}
}

@article{wise2000gamma,
  title={The $\gamma$-tubulin Gene Family in Humans},
  author={Wise, Dawnne O'Neal and Krahe, Ralf and Oakley, Berl R},
  journal={Genomics},
  volume={67},
  number={2},
  pages={164--170},
  year={2000},
  publisher={Elsevier}
}

@article{ettinger2010non,
  title={Non-small Cell Lung Cancer},
  author={Ettinger, David S and Akerley, Wallace and Bepler, Gerold and Blum, Matthew G and Chang, Andrew and Cheney, Richard T and Chirieac, Lucian R and D'Amico, Thomas A and Demmy, Todd L and Ganti, Apar Kishor P and others},
  journal={Journal of the National Comprehensive Cancer Network},
  volume={8},
  number={7},
  pages={740--801},
  year={2010},
  publisher={Harborside Press, LLC}
}

@article{pirker2009cetuximab,
  title={Cetuximab Plus Chemotherapy in Patients with Advanced Non-small-cell Lung Cancer ({FLEX}): {A}n Open-label Randomised Phase {III} Trial},
  author={Pirker, Robert and Pereira, Jose R and Szczesna, Aleksandra and Von Pawel, Joachim and Krzakowski, Maciej and Ramlau, Rodryg and Vynnychenko, Ihor and Park, Keunchil and Yu, Chih-Teng and Ganul, Valentyn and others},
  journal={The Lancet},
  volume={373},
  number={9674},
  pages={1525--1531},
  year={2009},
  publisher={Elsevier}
}

@article{qian2011performance,
  title={Performance Guarantees for Individualized Treatment Rules},
  author={Qian, Min and Murphy, Susan A},
  journal={The Annals of Statistics},
  volume={39},
  number={2},
  pages={1180},
  year={2011},
  publisher={NIH Public Access}
}

@article{murphy2005generalization,
  title={A Generalization Error for {Q}-learning},
  author={Murphy, Susan A},
  journal={Journal of Machine Learning Research},
  volume={6},
  number={Jul},
  pages={1073--1097},
  year={2005}
}

@article{zhang2015using,
  title={Using Decision Lists to Construct Interpretable and Parsimonious Treatment Regimes},
  author={Zhang, Yichi and Laber, Eric B and Tsiatis, Anastasios and Davidian, Marie},
  journal={Biometrics},
  volume={71},
  number={4},
  pages={895--904},
  year={2015},
  publisher={Wiley Online Library}
}

@phdthesis{wu2016set,
  title={Set Valued Dynamic Treatment Regimes},
  author={Wu, Tianshuang},
  year={2016},
  school={The University of Michigan}
}

@article{deseq2,
  title={Moderated Estimation of Fold Change and Dispersion for {RNA}-seq Data with {DESeq2}},
  author={Love, Michael I and Huber, Wolfgang and Anders, Simon},
  journal={Genome Biology},
  volume={15},
  number={12},
  pages={550},
  year={2014},
  publisher={BioMed Central}
}

@article{bourgon2010independent,
  title={Independent Filtering Increases Detection Power for High-throughput Experiments},
  author={Bourgon, Richard and Gentleman, Robert and Huber, Wolfgang},
  journal={Proceedings of the National Academy of Sciences},
  volume={107},
  number={21},
  pages={9546--9551},
  year={2010},
  publisher={National Acad Sciences}
}

@article{soneson2013comparison,
  title={A Comparison of Methods for Differential Expression Analysis of {RNA}-seq Data},
  author={Soneson, Charlotte and Delorenzi, Mauro},
  journal={BMC Bioinformatics},
  volume={14},
  number={1},
  pages={91},
  year={2013},
  publisher={BioMed Central}
}

@article{rashid2014differential,
  title={Differential and Limited Expression of Mutant Alleles in Multiple Myeloma},
  author={Rashid, Naim U and Sperling, Adam S and Bolli, Niccolo and Wedge, David C and Van Loo, Peter and Tai, Yu-Tzu and Shammas, Masood A and Fulciniti, Mariateresa and Samur, Mehmet K and Richardson, Paul G and others},
  journal={Blood},
  volume={124},
  number={20},
  pages={3110--3117},
  year={2014},
  publisher={Am Soc Hematology}
}

@article{laber2015tree,
  title={Tree-based Methods for Individualized Treatment Regimes},
  author={Laber, EB and Zhao, YQ},
  journal={Biometrika},
  volume={102},
  number={3},
  pages={501--514},
  year={2015},
  publisher={Oxford University Press}
}

@article{fan2008sure,
  title={Sure Independence Screening for Ultrahigh Dimensional Feature Space},
  author={Fan, Jianqing and Lv, Jinchi},
  journal={Journal of the Royal Statistical Society: Series B},
  volume={70},
  number={5},
  pages={849--911},
  year={2008},
  publisher={Wiley Online Library}
}

@article{wang2017breast,
  title={Breast Tumors Educate the Proteome of Stromal Tissue in an Andividualized but Coordinated Manner},
  author={Wang, Xuya and Mooradian, Arshag D and Erdmann-Gilmore, Petra and Zhang, Qiang and Viner, Rosa and Davies, Sherri R and Huang, Kuan-lin and Bomgarden, Ryan and Van Tine, Brian A and Shao, Jieya and others},
  journal={Sci. Signal.},
  volume={10},
  number={491},
  pages={eaam8065},
  year={2017},
  publisher={American Association for the Advancement of Science}
}

@article{roife2016ex,
  title={Ex Vivo Testing of Patient-Derived Xenografts Mirrors the Clinical Outcome of Patients with Pancreatic Ductal Adenocarcinoma},
  author={Roife, David and Dai, Bingbing and Kang, Ya'an and Perez, Mayrim V Rios and Pratt, Michael and Li, Xinqun and Fleming, Jason B},
  journal={Clinical Cancer Research},
  volume={22},
  number={24},
  pages={6021--6030},
  year={2016},
  publisher={AACR}
}

@article{whittle2015patient,
  title={Patient-derived Xenograft Models of Breast Cancer and Their Predictive Power},
  author={Whittle, James R and Lewis, Michael T and Lindeman, Geoffrey J and Visvader, Jane E},
  journal={Breast Cancer Research},
  volume={17},
  number={1},
  pages={17},
  year={2015},
  publisher={BioMed Central}
}

@inproceedings{scott2014patient,
  title={Patient-derived Xenograft Models in Gynecological Malignancies},
  author={Scott, Clare L and Mackay, Helen J and Haluska Jr, Paul},
  booktitle={American Society of Clinical Oncology educational book/ASCO. American Society of Clinical Oncology. Meeting},
  pages={e258},
  year={2014},
  organization={NIH Public Access}
}

@article{rosfjord2014advances,
  title={Advances in Patient-derived Tumor Xenografts: {F}rom Target Identification to Predicting Clinical Response Rates in Oncology},
  author={Rosfjord, Edward and Lucas, Judy and Li, Gang and Gerber, Hans-Peter},
  journal={Biochemical Pharmacology},
  volume={91},
  number={2},
  pages={135--143},
  year={2014},
  publisher={Elsevier}
}

@article{therasse2000new,
  title={New Guidelines to Evaluate the Response to Treatment in Solid Tumors},
  author={Therasse, Patrick and Arbuck, Susan G and Eisenhauer, Elizabeth A and Wanders, Jantien and Kaplan, Richard S and Rubinstein, Larry and Verweij, Jaap and Van Glabbeke, Martine and van Oosterom, Allan T and Christian, Michaele C and others},
  journal={Journal of the National Cancer Institute},
  volume={92},
  number={3},
  pages={205--216},
  year={2000},
  publisher={Oxford University Press}
}

@article{zhang2012estimating,
  title={Estimating Optimal Treatment Regimes from a Classification Perspective},
  author={Zhang, Baqun and Tsiatis, Anastasios A and Davidian, Marie and Zhang, Min and Laber, Eric},
  journal={Stat},
  volume={1},
  number={1},
  pages={103--114},
  year={2012},
  publisher={Wiley Online Library}
}

@article{ma2016bayesian,
  title={Bayesian Predictive Modeling for Genomic Based Personalized Treatment Selection},
  author={Ma, Junsheng and Stingo, Francesco C and Hobbs, Brian P},
  journal={Biometrics},
  volume={72},
  number={2},
  pages={575--583},
  year={2016},
  publisher={Wiley Online Library}
}

@article{murphy2001marginal,
  title={Marginal Mean Models for Dynamic Regimes},
  author={Murphy, Susan A and van der Laan, Mark J and Robins, James M and Conduct Problems Prevention Research Group},
  journal={Journal of the American Statistical Association},
  volume={96},
  number={456},
  pages={1410--1423},
  year={2001},
  publisher={Taylor \& Francis}
}

@article{zhang2013robust,
  title={Robust Estimation of Optimal Dynamic Treatment Regimes for Sequential Treatment Decisions},
  author={Zhang, Baqun and Tsiatis, Anastasios A and Laber, Eric B and Davidian, Marie},
  journal={Biometrika},
  volume={100},
  number={3},
  pages={681--694},
  year={2013},
  publisher={Oxford University Press}
}

@article{orellana2010dynamic,
  title={Dynamic Regime Marginal Structural Mean Models for Estimation of Optimal Dynamic Treatment Regimes, Part I: Main Content},
  author={Orellana, Liliana and Rotnitzky, Andrea and Robins, James M},
  journal={The International Journal of Biostatistics},
  volume={6},
  number={2},
  year={2010}
}

@article{robins2008estimation,
  title={Estimation and Extrapolation of Optimal Treatment and Testing Strategies},
  author={Robins, James and Orellana, Liliana and Rotnitzky, Andrea},
  journal={Statistics in Medicine},
  volume={27},
  number={23},
  pages={4678--4721},
  year={2008},
  publisher={Wiley Online Library}
}

@article{rubin1978bayesian,
  title={Bayesian Inference for Causal Effects: The Role of Randomization},
  author={Rubin, D.B.},
  journal={The Annals of Statistics},
  pages={34--58},
  volume = {6}, 
  number = {1},
  year={1978},
  publisher={JSTOR}
}

@article{athey2016recursive,
  title={Recursive Partitioning for Heterogeneous Causal Effects},
  author={Athey, Susan and Imbens, Guido},
  journal={Proceedings of the National Academy of Sciences},
  volume={113},
  number={27},
  pages={7353--7360},
  year={2016},
  publisher={National Acad Sciences}
}

@article{wager2017estimation,
  title={Estimation and Inference of Heterogeneous Treatment Effects Using Random Forests},
  author={Wager, Stefan and Athey, Susan},
  journal={Journal of the American Statistical Association},
  volume={113},
	number={523},
  year={2018},
  publisher={Taylor \& Francis}
}

@article{imai2013estimating,
  title={Estimating treatment effect heterogeneity in randomized program evaluation},
  author={Imai, Kosuke and Ratkovic, Marc and others},
  journal={The Annals of Applied Statistics},
  volume={7},
  number={1},
  pages={443--470},
  year={2013},
  publisher={Institute of Mathematical Statistics}
}

@article{tibshirani2005sparsity,
  title={Sparsity and Smoothness via the Fused Lasso},
  author={Tibshirani, Robert and Saunders, Michael and Rosset, Saharon and Zhu, Ji and Knight, Keith},
  journal={Journal of the Royal Statistical Society: Series B},
  volume={67},
  number={1},
  pages={91--108},
  year={2005},
  publisher={Wiley Online Library}
}

@article{luedtke2016super,
  title={Super-learning of an Optimal Dynamic Treatment Rule},
  author={Luedtke, Alexander R and van der Laan, Mark J},
  journal={The International Journal of Biostatistics},
  volume={12},
  number={1},
  pages={305--332},
  year={2016},
  publisher={De Gruyter}
}

\end{document}